\documentclass{article}

\PassOptionsToPackage{numbers}{natbib}
\usepackage[final]{nips_2018}

\usepackage[utf8]{inputenc} %
\usepackage[T1]{fontenc}    %
\usepackage{hyperref}       %
\usepackage{url}            %
\usepackage{booktabs}       %
\usepackage{amsfonts}       %
\usepackage{nicefrac}       %
\usepackage{microtype}      %

\usepackage{multirow}
\usepackage[inline]{enumitem}

\usepackage{color}
\usepackage{amsmath, amsthm, amssymb, amsfonts}
\newtheorem{theorem}{Theorem}
\newtheorem{lemma}{Lemma}
\newtheorem{definition}{Definition}

\usepackage{bbm}
\usepackage{algorithm}
\usepackage[noend]{algpseudocode}
\usepackage{comment}
\usepackage{wrapfig}

\usepackage{caption}
\usepackage{subcaption}

\renewcommand{\citet}[1]{\citep{#1}}%

\DeclareMathOperator*{\argmin}{arg\,min}
\DeclareMathOperator*{\argmax}{arg\,max}

\usepackage[disable]{todonotes}

\newcommand{\COMMENTFORSPACE}[1]{}

\algnewcommand\algorithmicforeach{\textbf{for each}}
\algdef{S}[FOR]{ForEach}[1]{\algorithmicforeach\ #1\ \algorithmicdo}

\algnewcommand\algorithmicswitch{\textbf{switch}}
\algnewcommand\algorithmiccase{\textbf{case}}
\algnewcommand\algorithmicassert{\texttt{assert}}
\algnewcommand\Assert[1]{\State \algorithmicassert(#1)}%
\algdef{SE}[SWITCH]{Switch}{EndSwitch}[1]{\algorithmicswitch\ #1\ \algorithmicdo}{\algorithmicend\ \algorithmicswitch}%
\algdef{SE}[CASE]{Case}{EndCase}[1]{\algorithmiccase\ #1}{\algorithmicend\ \algorithmiccase}%
\algtext*{EndSwitch}%
\algtext*{EndCase}%

\newcommand{\Dc}{\mathcal{D}}

\newcommand{\Xc}{\mathcal{X}}
\newcommand{\Yc}{\mathcal{Y}}

\newcommand{\Rb}{\mathbb{R}}

\newcommand{\sinit}{v _{(0)}}
\newcommand{\binit}{b _{(0)}}

\newcommand{\V}{V}
\newcommand{\Vk}{V _k}
\newcommand{\E}{E}
\newcommand{\Ek}{E _k}
\newcommand{\T}{T}
\newcommand{\Tk}{T _k}

\newcommand{\Ex}{\mathbb E}

\newcommand{\mytitle}{Learning Beam Search Policies via Imitation Learning}

\title{\mytitle}

\author{
  Renato Negrinho$^1$ \qquad Matthew R. Gormley$^1$ \qquad Geoffrey J. Gordon$^{1,2}$ \\
  $^{1}$Machine Learning Department, Carnegie Mellon University \\
  $^{2}$Microsoft Research \\
  \texttt{\{negrinho,mgormley,ggordon\}@cs.cmu.edu}
}

\begin{document}

\maketitle

\begin{abstract}
Beam search is widely used for approximate decoding in
structured prediction problems.
Models often use a beam at test time
but ignore its existence at train time, and therefore do not
explicitly learn how to use the beam.
We develop an unifying meta-algorithm for learning beam search policies
using imitation learning.
In our setting, the beam is part of the model, and not
just an artifact of approximate decoding.
Our meta-algorithm captures existing learning algorithms and suggests new ones.
It also lets us show novel no-regret guarantees for learning beam search policies.
\end{abstract}

\section{Introduction}
\label{sec:intro}

Beam search is the dominant method for
approximate decoding in structured prediction tasks such as
machine translation~\cite{sutskever2014sequence}, speech recognition~\cite{graves2013speech},
image captioning~\cite{vinyals2015show}, and
syntactic parsing~\cite{weiss2015structured}.
Most models that use beam search at test time
ignore the beam at train time and instead
are learned via methods like likelihood maximization.
They therefore
suffer from two issues that we jointly address in this work:
(1) learning ignores the existence of the beam and
(2) learning uses only oracle trajectories.
These issues lead to mismatches between the train and
test settings that negatively affect performance.
Our work addresses these two issues simultaneously by
using imitation learning to develop novel beam-aware algorithms
with no-regret guarantees.
Our analysis is inspired by DAgger~\cite{ross2011reduction}.
Beam-aware learning algorithms use beam search at both
train and test time.
These contrast with common two-stage learning algorithms
that, first, at train time,
learn a probabilistic model via maximum likelihood,
and then, at test time, use beam search for approximate decoding.
The insight behind beam-aware algorithms is that,
if the model uses beam search at test time,
then the model should be learned using beam search at train time.
Resulting beam-aware methods run beam search at train time (i.e., roll-in)
to collect losses that are then used to update the model parameters.
The first proposed beam-aware algorithms are
perceptron-based,
updating the parameters either when the
best hypothesis does not score first in the beam~\cite{collins_incremental_2004},
or when it falls out of the beam~\cite{daume2005learning}.

While there is substantial prior work on beam-aware algorithms,
none of the existing algorithms
expose the learned model
to its own consecutive mistakes at train time.
When rolling in with the learned model,
if a transition leads to a beam without the correct hypothesis,
existing algorithms either stop
\cite{collins_incremental_2004, huang2012structured, andor2016globally}
or reset to a beam with the correct hypothesis
\cite{daume2005learning, xu2007learning,
  wiseman2016sequence}.\footnote{\citet{goyal2017continuous} take a
  different approach by training with a differentiable approximation
  of beam search, but decode with the standard (non-differentiable)
  search algorithm at test time.}
Additionally, existing beam-aware algorithms either
do not have theoretical guarantees or
only have perceptron-style guarantees~\cite{xu2007learning}.
We are the first to prove no-regret guarantees for an algorithm
to learn beam search policies.
Imitation learning algorithms,
such as DAgger~\cite{ross2011reduction},
leverage the ability to query an oracle
at train time to learn a model that is competitive (in the no-regret sense)
to the best model in hindsight.
Existing imitation learning algorithms
such as
SEARN~\cite{daume2009search-based},
DAgger~\cite{ross2011reduction}\footnote{
Scheduled sampling~\cite{bengio2015scheduled} is an instantiation of DAgger.},
AggreVaTe~\cite{ross2014reinforcement}, and
LOLS~\cite{chang2015learning},
execute the learned model at train time to collect data that is
then labeled by the oracle and used for retraining.
Nonetheless, these methods do not take the beam into account at train time,
and therefore do not learn to use the beam effectively at test time.

We propose a new approach to learn beam search policies
using imitation learning that
addresses these two issues.
We formulate the problem as learning a policy
to traverse the combinatorial search space of beams.
The learned policy is induced via a scoring function:
the neighbors of the elements of a beam are scored and
the top $k$ are used to form the successor beam.
We learn a scoring function to match
the ranking induced by the oracle costs of the neighbors.
We introduce training losses that capture this insight, among which
are variants of the weighted all pairs loss~\cite{beygelzimer2008machine}
and existing beam-aware losses.
As the losses we propose are differentiable with respect to the scores,
our scoring function can be learned using modern online optimization
algorithms, e.g. Adam~\cite{kingma2015adam}.

In some problems (e.g., sequence labeling and syntactic parsing)
we have the ability
to compute oracle completions and oracle completion costs for
non-optimal partial outputs.
Within our imitation learning framework, we can
use this ability to compute oracle completion costs
for the neighbors of the elements of a beam at train time to
induce an oracle that allows us to continue collecting supervision
after the best hypothesis falls out of the beam.
Using this oracle information, we are able to propose a
DAgger-like beam-aware algorithm with no-regret guarantees.

We describe our
novel learning algorithm as an
instantiation of a meta-algorithm for learning beam search policies.
This meta-algorithm sheds light into key design decisions that
lead to more performant algorithms, e.g., the
introduction of better training losses.
Our meta-algorithm
captures much of the existing literature on beam-aware methods
(e.g.,~\cite{daume2005learning, huang2012structured}),
allowing a clearer understanding of and comparison to existing
approaches, for example, by emphasizing that they arise from specific choices of
training loss function and data collection strategy, and by proving
novel regret guarantees for them.

Our contributions are:
an algorithm for learning beam search policies
(Section~\ref{sec:data_collection}) with accompanying regret
guarantees (Section~\ref{sec:theory}),
a meta-algorithm that captures much of the existing literature
(Section~\ref{sec:learning}), and new theoretical results for the
early update~\cite{collins_incremental_2004} and LaSO
\cite{daume2005learning} algorithms (Section~\ref{sec:main_arbitrary_dc_policies}).

\section{Preliminaries}
\label{sec:preliminaries}

\paragraph{Structured Prediction as Learning to Search}
We consider structured prediction in the learning to search framework
\cite{daume2009search-based,ross2011reduction}.
Input-output training pairs $D=\{(x_1, y_1), \ldots, (x_m, y_m)\}$ are
drawn according to a data generating distribution $\Dc$ jointly over an
input space $\Xc$ and an output space $\Yc$.
For each input $x \in \Xc$, there is an underlying search space
$G_x = (V_x, E_x)$ encoded as a directed graph with
nodes $V_x$ and edges $\E _x$.
Each output $y \in \Yc_x$ is encoded as a terminal node
in $G_x$, where $\mathcal Y _x \subseteq \mathcal Y$ is the
set of valid structured outputs for $x$.

In this paper,
we deal with stochastic policies $\pi : V _x \to \Delta(V _x)$,
where $\Delta(V _x)$ is the set of probability distributions over
nodes in $V _x$.
\COMMENTFORSPACE{
The distributions computed by a policy $\pi : V_x \to \Delta(V _x)$ must
respect the structure of $G_x$, meaning that for all $v \in V_x$,
the support of $\pi(v)$ must be contained in the neighborhood of $v$.
}
(For convenience and brevity of presentation,
we make our policies deterministic later in the paper through the introduction
of a tie-breaking total order over the elements of $V_x$,
but our arguments and theoretical results hold more generally.)
The goal is to learn a
stochastic policy $\pi(\cdot, x, \theta) : \V _x \to \Delta(\V _x)$
parametrized by $\theta \in \Theta \subseteq \mathbb R ^p$
that traverses the induced search spaces,
generating outputs with small expected cost;
i.e., ideally, we would want to minimize
\begin{equation}
c(\theta) =
  \mathbb E _{(x, y) \sim \mathcal D }
  \mathbb E _{\hat y \sim \pi(\cdot, x, \theta)}
    c _{x, y} (\hat y),
  \label{eq:opt_prob}
\end{equation}
where $c _{x, y} : \mathcal Y _x \to \mathbb R$ is the cost function
comparing the ground-truth labeling $y$ to the predicted labeling $\hat y$.
We are not able to optimize directly the loss in Equation~\eqref{eq:opt_prob},
but we are able to find a mixture of policies
$\theta _1, \ldots, \theta _m$, where $\theta _t \in \Theta$
for all $t \in [m]$, that is competitive
with the best policy in $\Theta$
in the distribution of trajectories induced by the mixture of
$\theta _1, \ldots, \theta _m$.
We use notation $\hat y \sim \pi(\cdot, x, \theta)$ to mean that $\hat y$
is generated by sampling a trajectory $v _1, \ldots, v _h$ on $G _x$
by executing policy $\pi(\cdot, x, \theta)$, and returning
the labeling $\hat y \in \mathcal Y$ associated with terminal node $v _h \in T$.
The search spaces, cost functions and policies depend on
$x \in \mathcal X$ or $(x, y) \in \mathcal X \times \mathcal Y$---in
the sequel, we omit indexing by example for conciseness.

\paragraph{Search Space, Cost, and Policies}
Each example
$(x, y) \in \mathcal X \times \mathcal Y$ induces a
search space $G = (V, E)$ and a cost function $c : \mathcal Y \to \mathbb R$.
For all $v \in \V$, we introduce
its set of neighbors
$N _v = \{v' \in \V \mid (v, v') \in \E \}$.
We identify a single initial node $\sinit \in V$.
We define the set of terminal nodes
$\T = \{v \in \V \mid N _v = \emptyset\}$.
We assume without loss of generality that all nodes are reachable from $\sinit$
and that all nodes have paths to terminal nodes.
For clarity of exposition,
we assume that $G$ is a tree-structured directed graph
where all terminals nodes are at distance $h$ from the root $\sinit$.\footnote{
    We describe in Appendix~\ref{sec:ConversionToTreeStructured} how to
convert a directed graph search space to a tree-structured one with all terminals
at the same depth.}

Each terminal node $v \in T$ corresponds to a complete output $y \in \Yc$,
which can be compared to the ground-truth $y^* \in \Yc$
via a cost function $c : T \to \mathbb R$ of interest
(e.g.,
Hamming loss in sequence labeling or negative BLEU score~\cite{papineni2002bleu}
in machine translation).
We define the optimal completion cost function
$c^* : \V\to \mathbb R$, which computes
the cost of the best terminal node reachable from $v \in V$ as
$c^*(v) = \min _{v' \in T _{v} } c(v')$,
where $T_v$ is the set of terminal nodes reachable from $v$.

The definition of $c ^*: V \to \mathbb R$ naturally gives rise to an
oracle policy $\pi^*(\cdot, c^*) : \V \to \Delta(\V)$.
At $v \in \V$, $\pi^*(v, c^*)$ can be any fixed
distribution (e.g., uniform or one-hot) over
$\argmin_{v' \in N_v} c^*(v')$.
For any state $v \in \V$, executing $\pi^*(\cdot, c^*)$
until arriving at a terminal node
achieves the lowest possible cost for completions of $v$.

At $v \in \V$, a greedy policy $\pi : V \to \Delta(V)$ induced
by a scoring function
$s : V \to \mathbb R$ computes a fixed
distribution
$\pi(v, \theta)$ over $\argmax_{v' \in N_v} s(v', \theta)$.
When multiple elements are tied with the same highest score, we
can choose an arbitrary distribution over them.
If there is a single highest scoring element, the policy is deterministic.
In this paper, we assume the existence of a total order over the elements
of $V$ that is used for breaking ties induced by a scoring function.
The tie-breaking total ordering allows us to talk
about a particular unique ordering, even when ties occur.
The oracle policy $\pi^*(\cdot, c^*) : V \to \Delta(V)$ can be thought
as being induced by the scoring function $-c^*: V \to \mathbb R$.

\COMMENTFORSPACE{
\paragraph{Running Example}
We use part-of-speech tagging as a simple running example: an
input $x_{1:l} \in \Xc$ is a sequence of words of length
$l$ and an output $y_{1:l} \in \Yc$ is a sequence of part-of-speech tags of the
same length.
A terminal node $v \in T$ is a complete tagging $y_{1:l} \in \Yc$.
For each non-terminal node $v = y _{1:j} \in \V \setminus T$,
the set of neighbors $N _v$ encodes extending partial tagging $y _{1:j}$
by tagging token $x _{j + 1}$ with
each possible part-of-speech tag.
As part-of-speech tagging does not impose constraints on tag sequences, all
non-terminals have the same number of neighbors, i.e., for all
$v, v' \in \V \setminus T$, $|N _v| = | N _{v'}|$.
Search spaces for part-of-speech tagging are tree-structured.
A common cost function is the Hamming loss, yielding the optimal cost
$c^*(y_{1:j}) = \sum _{i = 1} ^{j} \mathbbm{1} [ y _i \neq y ^* _i ]$,
i.e., the number of misclassifications in the prefix.
In this case, the oracle policy $\pi^*(\cdot, c^*) : \V \to \V$ is deterministic
and extends
a partial tagging $y_{1:j}$ with the
next ground-truth tag $y ^*_{j +  1}$,
for all $i \in [l - 1]$.
}

\section{Beam search}
\label{sec:beam_search}

\begin{wrapfigure}[17]{R}{0.43\textwidth}
  \begin{minipage}{0.43\textwidth}
\vspace{-2em}
\begin{algorithm}[H]
  \begin{algorithmic}[1] %
    \Function{BeamSearch}{$G, k, \theta$}
      \State $b \gets \{\sinit\} \equiv \binit$ %
      \While{$\textproc{Best}(b, 1, s(\cdot, \theta)) \notin T$}
        \State $b \gets \textproc{Policy}(G, b, k, s(\cdot, \theta))$
      \EndWhile
      \State \textbf{return} $\textproc{Best}(b, 1, s(\cdot, \theta))$ %
    \EndFunction
    \vspace{.5em}
    \hrule
    \vspace{.5em}
    \Function{Policy}{$G, b, k, f$}
      \State Let $A_b = \cup_{v \in b} N_v$ %
      \State {\bf return} $\textproc{Best}(A_b, k, f)$ %
    \EndFunction
    \vspace{.5em}
    \hrule
    \vspace{.5em}
    \Function{Best}{$A, k, f$}
      \State Let $A = \{v_1, \ldots, v_n\}$ be ordered %
      \State $\quad$ such that $f(v_1) \geq \cdots \geq f(v_n)$
      \State Let $k' = \min(k, n)$
      \State \textbf{return} ${v_1, \ldots, v_{k'}}$
    \EndFunction
  \end{algorithmic}
  \caption{Beam Search}
  \label{alg:beam_search}
\end{algorithm}
\end{minipage}
\end{wrapfigure}

\paragraph{Beam Search Space}
\label{sec:beam_search_space}

Given a search space $G$, we construct
its beam search space $G_k = (\Vk, \Ek)$,
where $k \in \mathbb N$ is the maximum beam capacity.
$\Vk$ is the set of possible
beams that can be formed along the search process,
and $\Ek$ is the set of possible beam transitions.
Nodes $b \in \Vk$ correspond to
nonempty sets of nodes of $\V$ with
size upper bounded by $k$, i.e.,
$b = \{ v _1, \ldots, v _{ |b| } \}$ with
$1 \leq |b| \leq k$
and $v _i \in \V$ for all $i \in [|b|]$.
The initial beam state $\binit \in \Vk$
is the singleton set with the
initial state $\sinit \in \V$.
Terminal nodes in $T _k$
are singleton sets with a single terminal node $v \in T$.
For $b \in \Vk$,
we define $A _b = \cup _{v \in b} N _{v}$, i.e.,
the union of the neighborhoods of the elements in $b$.

Algorithm~\ref{alg:beam_search} describes the beam search variant
used in our paper.
In this paper, all elements in the beam
are simultaneously expanded when transitioning.
It is possible to define different beam search space variants, e.g.,
by considering different expansion strategies or by handling terminals
differently (in the case where terminals can be at different depths).
The arguments developed in this paper can be extended to those variants
in a straightforward manner.

\paragraph{Beam Costs}
\label{sec:beam_costs}

We define the cost of a beam to be the cost of its lowest cost element, i.e.,
we have $c^* : \Vk \to \mathbb R$ and,
for $b \in \Vk$,
$c^*(b) = \min _{v \in b} c^*(v)$.
We define the beam transition cost function $c : \Ek \to \mathbb R$ to be
$c(b, b') = c^*(b') - c^*(b)$,
for $(b, b') \in E _k$, i.e., the difference in cost between the
lowest cost element in $b'$ and the lowest cost element in $b$.

A cost increase occurs on a transition
$(b, b') \in \Ek$ if $c^*(b') > c^*(b)$, or equivalently,
$c(b, b') > 0$, i.e., $b'$ dropped all the
lowest cost neighbors of the elements of $b$.
For all $b \in \Vk$, we define
$N ^* _b = \{ b' \in N _b \mid c(b, b') = 0 \}$,
i.e., the set of beams neighboring $b$
that do not lead to cost increases.
We will significantly overload notation, but usage will be clear from
context and argument types, e.g., when referring to
$c^* : V \to \mathbb R$ and $c^* : V_k \to \mathbb R$.

\begin{wrapfigure}[27]{R}{0.55\textwidth}
  \vspace{-2em}
\begin{minipage}{0.55\textwidth}
\begin{algorithm}[H]
    \begin{algorithmic}[1] %
      \Function{Learn}{$D, \theta _1, k$}
          \For{each $t \in [ |D| ]$}
            \State Induce $G$ using $x _t$
            \State Induce $s(\cdot, \theta _t): V \to \mathbb R$ using $G$ and $\theta _t$
            \State Induce $c ^*: V \to \mathbb R$ using $(x _t, y _t$)
            \State $b _{1:j} \gets \textproc{BeamTrajectory}(G, c^*, s(\cdot, \theta _t), k)$
            \State Incur losses $\ell(\cdot, b _1), \ldots, \ell(\cdot, b_{j - 1})$
            \State Compute $\theta _{t + 1}$ using
              $\sum_{i = 1} ^{j - 1} \ell(\cdot, b _i)$, e.g., by SGD or Adam
          \EndFor
        \State {\bf return} best $\theta_t$ on validation
      \EndFunction
    \vspace{.5em}
    \hrule
    \vspace{.5em}
    \Function{BeamTrajectory}{$G, c^*, f, k$}
    \State $b _1 \gets \{\sinit\} \equiv \binit$
    \State $j = 1$
    \While{$\textproc{Best}(b _j, 1, f) \notin T$}
        \If{strategy is oracle}
          \State $b _{j + 1} \gets \textproc{Policy}(G, b_j, k, -c ^*)$
        \Else
          \State $b _{j + 1} \gets \textproc{Policy}(G, b_j, k, f)$
          \If{$c^*(b _{j + 1}) > c^*(b _j)$}
            \If{strategy is stop}
              \State break
            \EndIf
            \If{strategy is reset}
              \State $b _{j + 1} \gets \textproc{Policy}(G, b _j, 1, -c^*)$
            \EndIf
          \EndIf
        \EndIf
        \State $j \gets j + 1$
    \EndWhile
    \State \textbf{return} $b _{1:j}$
    \EndFunction
    \end{algorithmic}
    \caption{Meta-algorithm}
    \label{alg:meta_algorithm}
  \end{algorithm}
\end{minipage}
\end{wrapfigure}

\paragraph{Beam Policies}
\label{sec:beam_policies}

Let $\pi : \Vk \to \Delta(\Vk)$ be a policy induced
by a scoring function
$f : V \rightarrow \Rb$.
To sample $b' \sim \pi(b)$
for a beam $b \in \Vk$, form $A _b$, and
compute scores $f(v)$ for all $v \in A _b$; let $v _1, \ldots, v _n$ be the
elements of $A _b$ ordered such that
$f(v _1) \geq \ldots \geq f(v _n)$;
if $v _1 \in T$,
$b' = \{ v _1 \}$; if
$v _1 \not \in T$, let $b'$ pick the $k$ top-most
elements from $A _b \setminus T$.
At $b \in V_k$,
if there are many orderings that sort the scores of the
elements of $A_b$,
we can choose a single one deterministically or sample one stochastically;
if there is a single such ordering,
the policy $\pi : V_k \to \Delta(V_k)$ is deterministic at $b$.

For each $x \in \mathcal X$, at train time,
we have access to the optimal path cost function
$c^*: V \to \mathbb R$, which induces the oracle policy
$\pi ^*(\cdot, c^*): \Vk \to \Delta(\Vk)$.
At a beam $b$, a successor beam $b' \in N_b$ is optimal if $c^*(b') = c^*(b)$,
i.e., at least one neighbor with the smallest possible cost was included in $b'$.
The oracle policy $\pi ^*(\cdot, c^*) : V_k \to \Delta(V_k)$ can be seen as using
scoring function $-c^*: V_k \to \mathbb R$ to transition in the beam
search space $G_k$.
\COMMENTFORSPACE{
i.e., for $b \in V_k$, $\pi ^*(b, c^*)$ is a distribution
over $\argmin_{b' \in N_b} c^*(b')$.
}

\COMMENTFORSPACE{
\paragraph{Running Example}
In our part-of-speech tagging example, a beam $b \in V_k$ at step $j$ is a set
of length-$j$ partial taggings $\{y_{1,1:j}, \ldots, y_{k,1:j}\}$.
We transition to a successor beam $b'$ by extending the length
of each tagging
$y_{i,1:j}$ by one tag and choosing the top $k$ taggings of length
$j+1$ according to the scoring function.
For all $b \in \Vk$,
$c^*(b)$ is the smallest number of wrong part-of-speech tags
among the hypothesis in $b$.
}

\section{Meta-Algorithm}
\label{sec:learning}
Our goal is to learn a policy $\pi(\cdot, \theta) : \Vk
\rightarrow \Delta(\Vk)$ induced by a scoring function $s(\cdot, \theta) : V
\rightarrow \Rb$ that achieves small expected cumulative transition
cost along the induced trajectories.
Algorithm~\ref{alg:meta_algorithm} presents our meta-algorithm in detail.
\COMMENTFORSPACE{
In Algorithm~\ref{alg:meta_algorithm}, we repeatedly sample an $(x,y)$ pair
from our training distribution. We run a beam search according to our
current parameter vector $\theta$. During the search, we accumulate
loss functions (one per beam transition) based on $(x, y)$, and the
current beam state. At the end of each search, we update $\theta$
according to the accumulated losses.
}
Instantiating our meta-algorithm requires choosing
both a surrogate training loss function (Section~\ref{sec:surrogate_losses})
and a data collection strategy (Section~\ref{sec:data_collection}).
Table~\ref{tab:priorWork} shows how existing algorithms can be obtained as
instances of our meta-algorithm with specific choices of loss function, data collection
strategy, and beam size.

\subsection{Surrogate Losses}
\label{sec:surrogate_losses}

\paragraph{Insight}
In the beam search space, a prediction $\hat y \in \mathcal Y _x$
for $x \in \mathcal X$ is generated by running
$\pi(\cdot, \theta)$ on $G_k$.
This yields a beam trajectory $b _{1:h}$, where $b _1 = \binit$ and
$b _h \in \Tk$.
We have
\begin{align}
c(\theta)
  = \mathbb E _{(x, y) \sim \mathcal D }
    \mathbb E _{\hat y \sim \pi(\cdot, \theta)}
      c (\hat y)
  = \mathbb E _{(x, y) \sim \mathcal D }
      \mathbb E _{b _{1:h} \sim \pi(\cdot, \theta)}
        c ^*(b _h).
\end{align}
The term $c^*(b_h)$ can be written in a telescoping manner as
\begin{align}
  c ^*(b _h) =
    c^*(b _1) +
      \sum _{i = 1} ^{h - 1} c ( b _i, b _{i + 1} ).
\end{align}
As $c^*(b _1)$ depends on an example
$(x, y) \in \mathcal X \times \mathcal Y$,
but not on the parameters $\theta \in \Theta$,
the set of minimizers of $c : \Theta \to \mathbb R$ is the same as the
set of minimizers of
\begin{align}
  c'(\theta)
    &= \mathbb E _{(x, y) \sim \mathcal D }
        \mathbb E _{b _{1:h} \sim \pi(\cdot, \theta)}
          \left(
            \sum _{i = 1} ^{h - 1} c ( b _i, b _{i + 1} )
          \right).
          \label{eq:opt_beam_last}
\end{align}
It is not easy to minimize the cost function in Equation~\eqref{eq:opt_beam_last}
as, for example,
$c(b, \cdot) : V _k \to \mathbb R$ is combinatorial.
To address this issue, we observe the following by using linearity of expectation
and the law of iterated expectations to decouple the term in the sum over the trajectory:
\begin{align}
  \mathbb E _{b _{1:h} \sim \pi(\cdot, \theta)}
    \left(
      \sum _{i = 1} ^{h - 1} c ( b _i, b _{i + 1} )
    \right)
   &=
      \sum _{i = 1} ^{h - 1}
        \mathbb E _{b _i \sim d _{\theta, i} }
          \mathbb E _{b _{i + 1} \sim \pi( b_i, \theta) }
            c ( b _i, b _{i + 1} ) \nonumber \\
   &=
    \mathbb E _{b _{1:h} \sim \pi(\cdot, \theta)}
      \left(
        \sum _{i = 1} ^{h - 1}
          \mathbb E _{b' \sim \pi( b_i, \theta) }
            c ( b _i, b' )
      \right),
\end{align}
where $d _{\theta, i}$ denotes the distribution over beams in $V _k$
that results from following $\pi(\cdot, \theta)$ on $G _k$ for $i$ steps.
We now replace $\Ex _{ b' \sim \pi(b, \cdot) } c(b, b') : \Theta \to \mathbb R$ by a
surrogate loss function
$\ell(\cdot, b) : \Theta \to \mathbb R$
that is differentiable with respect to the
parameters $\theta \in \Theta$,
and where
$\ell(\theta, b)$ is a surrogate loss for the expected cost increase
incurred by following policy $\pi(\cdot, \theta)$ at beam $b$ for one step.

Elements in $A _b$ should be scored in a way that allows
the best elements to be kept in the beam.
Different surrogate losses arise from which elements we concern ourselves with, e.g.,
all the top $k$ elements in $A _b$ or simply one of the best elements in $A _b$.
Surrogate losses are then large when the scores
lead to discarding desired elements in $A _b$, and
small when the scores lead to comfortably keeping the desired elements in $A_b$.

\paragraph{Surrogate Loss Functions}
The following additional notation allows us to define losses
precisely.
Let $A _b = \{ v _1, \ldots, v _n \}$ be an arbitrary ordering of
the neighbors of the elements in $b$.
Let $c = c _1, \ldots, c _n$ be the corresponding costs, where
$c_i = c^*(v_i)$ for all $i \in [n]$, and
$s = s _1, \ldots, s _n$ be the
corresponding scores,
where $s _i = s(v _i, \theta)$ for all $i \in [n]$.
Let $\sigma ^*: [n] \to [n]$ be a permutation such that
$c _{\sigma ^*(1)} \leq \ldots \leq c _{\sigma ^*(n)}$, i.e.,
$v _{ \sigma ^*(1) }, \ldots, v _{ \sigma ^*(n) }$
are ordered in increasing order of cost.
Note that $c^*(b) = c _{ \sigma ^* (1) }$.
Similarly, let $\hat \sigma : [n] \to [n]$ be a permutation
such that $s _{\hat \sigma (1)} \geq \ldots \geq s _{\hat \sigma (n)}$, i.e.,
$v _{ \hat \sigma(1) }, \ldots, v _{ \hat \sigma(n) }$
are ordered in decreasing order of score.
We assume unique $\sigma^* : [n] \to [n]$ and $\hat \sigma: [n] \to [n]$ for
simplifying the presentation of the loss functions
(which can be guaranteed via the tie-breaking total order on $V$).
In this case, at $b \in V_k$, the successor beam $b' \in N_b$ is uniquely
determined by the scores of the elements of $A _b$.

For each $(x, y) \in \mathcal X \times \mathcal Y$,
the corresponding cost function $c^* : \V \to \mathbb R$
is independent of the parameters $\theta \in \Theta$.
We define a loss function $\ell(\cdot, b) : \Theta \to \mathbb R$
at a beam $b \in V_k$ in terms of the oracle
costs of the elements of $A_b$.
We now introduce some well-motivated surrogate loss functions.
Perceptron and large-margin inspired losses have been used in
early update~\cite{collins_incremental_2004},
LaSO~\cite{daume2005learning}, and BSO~\cite{wiseman2016sequence}.
We also introduce two log losses.

\paragraph{perceptron (first)}
  Penalizes the lowest cost element in $A _b$
  not being put at the top of the beam.
  When applied on the first cost increase, this is
  equivalent to an ``early update''~\cite{collins_incremental_2004}.
  \begin{align}
    \ell(s, c) =
            \max
              \left( 0, s _{ \hat \sigma (1) } - s _{ \sigma ^*(1) } \right).
      \label{eq:perceptron_first}
  \end{align}
\paragraph{perceptron (last)}
  Penalizes the lowest cost element in $A _b$ falling out of the beam.
  \begin{align}
    \ell(s, c) =
          \max
            \left( 0, s _{ \hat \sigma (k) } - s _{ \sigma ^*(1) } \right).
    \label{eq:perceptron_last}
  \end{align}
\paragraph{margin (last)}
  Prefers the lowest cost element to be scored higher than
  the last element in the beam by a margin.
  This yields updates that are similar but not identical to the approximate
  large-margin variant of LaSO~\cite{daume2005learning}.
  \begin{align}
    \ell(s, c) =
          \max
            \left( 0, 1 + s _{ \hat \sigma (k) } - s _{ \sigma ^*(1) } \right)
    \end{align}
\paragraph{cost-sensitive margin (last)}
    Weights the margin loss by the cost
    difference between the lowest cost element and the last element in
    the beam. When applied on a LaSO-style cost increase, this is
  equivalent to the BSO update of \citet{wiseman2016sequence}.
  \begin{align}
    \ell(s, c) =
        (c _{\hat \sigma (k) } - c _{ \sigma ^* (1) })
          \max
            \left( 0, 1 + s _{ \hat \sigma (k) } - s _{ \sigma ^*(1) } \right).
    \label{eq:cs_margin_last}
  \end{align}
  \paragraph{upper bound}
  Convex upper bound to the expected beam transition cost,
  $\mathbb E _{b' \sim \pi(b, \cdot)} c(b, b') : \Theta \to \mathbb R$,
  where $b'$ is induced by the scores $s \in \mathbb R ^n$.
  \begin{align}
    \ell(s, c) =
    \max
      \left(
        0, \delta _{k + 1}, \ldots, \delta _{n}
      \right)
      \label{eq:cvx_upper_bound}
  \end{align}
  where $\delta _j = ( c _{ \sigma ^*(j) } - c _{ \sigma ^*(1)} )
  ( s _{\sigma ^*(j)} - s _{\sigma ^*(1)} + 1 )$
  for $j \in \{k + 1, \ldots, n\}$.
  Intuitively, this loss imposes a cost-weighted margin between the
  best neighbor $v _{ \sigma^* (1) } \in A _b$ and the neighbors
  $v _{\sigma^*(k + 1) }, \ldots,  v _{\sigma^*(n) } \in A _b$ that ought
  not to be included in the best successor beam $b'$.
  We prove in Appendix~\ref{sec:upper_bound} that this loss
  is a convex upper bound for the expected beam transition cost.

\paragraph{log loss (beam)}
  Normalizes only over the top $k$
  neighbors of a beam according to the scores $s$.
  \begin{align}
        \ell(s, c)
          &= - s _{ \sigma ^*(1) } + \log
          \left(
              \sum _{ i \in I }  \exp ( s _i )
          \right) ,
      \label{eq:andor}
  \end{align}
  where $I =\{ \sigma ^*(1), \hat \sigma(1), \ldots, \hat \sigma(k) \}$.
  The normalization is only over the correct element $v _{\sigma ^*(1)}$ and the
  elements included in the beam.
  The set of indices $I \subseteq [n]$ encodes the fact that
  the score vector $s \in \mathbb R^n$ may not place
  $v _{\sigma^*(1)}$
  in the top
  $k$, and therefore it has to also be included in that case.
  This loss is used in \citet{andor2016globally}, albeit introduced differently.
\paragraph{log loss (neighbors)}
    Normalizes over all elements in $A_b$.
    \begin{align}
          \ell(s, c)
          &= - s _{ \sigma ^*(1) } + \log
          \left(
              { \sum _{ i = 1} ^n \exp ( s _i ) }
          \right)
      \label{eq:log_all}
  \end{align}

\paragraph{Discussion}
The losses here presented directly capture
the purpose of using a beam for prediction---ensuring that the best
hypothesis stays in the beam, i.e., that,
at $b \in V_k$, $v _{\sigma^*(1)} \in A _b$
is scored sufficiently high to be included in the
successor beam $b' \in N _b$.
If full cost information is not accessible,
i.e., if are not able to evaluate $c^* : V \to \mathbb R$
for arbitrary elements in $V$,
it is still possible to use a subset of these losses, provided that
we are able to identify the lowest cost element among the neighbors of a beam,
i.e., for all $b \in \Vk$, an element $v \in A _b$,
such that $c^*(v) = c^*(b)$.

While certain losses do not
appear beam-aware
(e.g., those in Equation~\eqref{eq:perceptron_first} and Equation~\eqref{eq:log_all}),
it is important to keep in mind that all
losses are collected by executing a policy on the beam search space $G_k$.
Given a beam $b \in V _k$, the score vector $s \in \mathbb R^n$ and
cost vector $c \in \mathbb R^n$ are defined for the elements of $A_b$.
The losses incurred depend on the specific beams visited.
Losses in Equation~\eqref{eq:perceptron_first}, \eqref{eq:cvx_upper_bound}, and
\eqref{eq:log_all} are convex. The remaining losses are non-convex.
For $k = 1$, we recover well-known losses, e.g., loss in
Equation~\eqref{eq:log_all} becomes a simple log loss over the neighbors
of a single node, which is precisely the loss used in typical log-likelihood
maximization models;
loss in Equation~\eqref{eq:perceptron_last} becomes a perceptron loss.
In Appendix~\ref{sec:convex_cosiderations} we discuss
convexity considerations for different types of losses.
In Appendix~\ref{sec:additional_losses},
we present additional losses and expand on
their connections to existing work.

\subsection{Data Collection Strategy}
\label{sec:data_collection}

Our meta-algorithm requires choosing a train time policy
$\pi : \Vk \rightarrow \Delta(\Vk)$ to traverse the beam search
space $G_k$ to collect supervision.
Sampling a trajectory to collect training
supervision is done by \textproc{BeamTrajectory}
in Algorithm~\ref{alg:meta_algorithm}.

\paragraph{oracle} Our simplest policy follows the oracle policy
$\pi^* : \Vk \rightarrow \Delta(\Vk)$ induced by the optimal completion cost function
$c^* : V \rightarrow \Rb$ (as in Section~\ref{sec:beam_policies}).
Using the terminology of Algorithm~\ref{alg:beam_search}, we can write
$\pi^*(b, c ^*) = \textproc{Policy}(G, b, k, -c^*)$. This policy
transitions using the negated sorted costs of the elements in $A _b$
as scores.
\COMMENTFORSPACE{
For $b \in \Vk$,
while transitioning using the negated sorted costs of the elements in $A _b$
as scores is the obvious choice,
any beam $b' \in N _b ^*$ is a valid optimal neighboring beam, and therefore,
$\pi^*(b, c^*)$ can be any distribution over $N _b ^*$.
}

The oracle policy does not address the distribution mismatch problem.
At test time, the learned policy will make mistakes and
visit beams for which it has not collected supervision at train time,
leading to error compounding.
Imitation learning tells us that
it is necessary to collect supervision at train time with the learned policy
to avoid error compounding at test time~\cite{ross2011reduction}.

We now present data collection strategies that use the learned policy.
For brevity, we only cover the case where the learned policy is always used
(except when the transition leads to a cost-increase), and leave
the discussion of additional possibilities
(e.g., probabilistic interpolation of learned and oracle policies)
to Appendix~\ref{sec:arbitrary_dc_policies}.
When an edge $(b, b') \in \Ek$
incurring cost increase is traversed, different strategies are possible:
    \paragraph{stop}  Stop
    collecting the beam trajectory.
    The last beam in the trajectory is $b'$, i.e., the beam on which we
    arrive in the transition that led to a cost increase.
  This data collection
  strategy is used in structured perceptron training with early
  update~\cite{collins_incremental_2004}.
\paragraph{reset} Reset the beam to contain
  only the best state as defined by the optimal completion cost function:
  $b' = \textproc{Best}(b, 1, -c^*)$.
  In the subsequent steps of the policy, the beam grows back
  to size $k$.
  LaSO~\cite{daume2005learning} uses this data collection strategy.
  Similarly to the oracle data collection strategy, rather than committing
  to a specific $b' \in N ^*_b$, we can sample $b' \sim \pi ^*(b, c^*)$
  where $\pi ^*(b, c^*)$ is any distribution over $N_b^*$.
  The reset data collection strategy collects beam trajectories where
  the oracle policy $\pi $ is executed conditionally, i.e., when the
  roll-in policy $\pi(\cdot, \theta_t)$ would lead to a cost increase.
  \paragraph{continue} We can ignore the cost
  increase and continue following policy $\pi_t$. This is the
  strategy taken by DAgger~\cite{ross2011reduction}.
  The continue data collection strategy has not been considered in the beam-aware
  setting, and therefore it is a novel contribution of our work.
  Our stronger theoretical guarantees apply to this case.

  \begin{table}[bt]
    \caption{Existing and novel beam-aware algorithms as instances of our meta-algorithm.
      Our theoretical guarantees require the existence of a
      deterministic no-regret online learning algorithm for the
      resulting problem.}
    \label{tab:priorWork}
    \centering
    \vspace{.5em}
    \begin{tabular}{l|lll}
      \toprule
      {\bf Algorithm} & \multicolumn{3}{|c}{\bf Meta-algorithm choices} \\
      & data collection & surrogate loss & $k$ \\
      \midrule
      log-likelihood & { oracle} & log loss (neighbors) & 1 \\
      {\sc DAgger}~\cite{ross2011reduction} & { continue} & log loss (neighbors) & 1 \\
      early update~\cite{collins_incremental_2004} & { stop} & perceptron (first) & $>1$\\
      LaSO (perceptron)~\cite{daume2005learning} & { reset} & perceptron (first) & $>1$ \\
      LaSO (large-margin)~\cite{daume2005learning} & { reset} & margin (last) & $>1$ \\
      BSO~\cite{wiseman2016sequence} & { reset} & cost-sensitive margin (last) & $>1$ \\
      globally normalized~\cite{andor2016globally} & { stop} & log loss (beam)&$>1$ \\
      Ours & { continue} & [choose a surrogate loss] & $>1$ \\
      \bottomrule
    \end{tabular}
  \end{table}

\section{Theoretical Guarantees}
\label{sec:theory}

We state regret guarantees for learning beam search policies using the
continue, reset, or stop data collection strategies.
One of the main contributions of our work is framing the problem of learning
beam search policies in a way that allows us
to obtain meaningful regret guarantees.
Detailed proofs are provided in Appendix~\ref{sec:noregret}.
We begin by analyzing the continue collection strategy. As we will
see, regret guarantees are stronger for continue than for stop
or reset.

No-regret online learning algorithms have an important role in the proofs
of our guarantees.
Let $\ell _1, \ldots, \ell _m$ be a sequence of loss functions with
$\ell _t : \Theta \to \mathbb R$ for all $t \in [m]$.
Let $\theta _1, \ldots, \theta _m$ be a sequence of iterates with
$\theta _t \in \Theta$ for all $t \in [m]$.
The loss function $\ell _t$ can be chosen according to an arbitrary rule
(e.g., adversarially).
The online learning algorithm chooses the
iterate $\theta _t$. Both $\ell _t$ and $\theta _t$  are chosen online, as functions
of loss functions $\ell _1, \ldots, \ell _{t - 1}$ and
iterates $\theta _1, \ldots, \theta _{t - 1}$.
\begin{definition}
  An online learning algorithm is no-regret if for any
  sequence of functions $\ell _1, \ldots, \ell _m$ chosen according to the
  conditions above we have
  \begin{align}
      \frac 1 m \sum _{t = 1} ^m \ell _t(\theta _t)
          - \min _{\theta \in \Theta}  \frac 1 m \sum _{t = 1} ^m
            \ell _t(\theta) = \gamma _m,
  \end{align}
  where $\gamma _m$ goes to zero as $m$ goes to infinity.
\label{th:no_regret_alg}
\end{definition}
Many no-regret online learning algorithms, especially for convex loss functions,
have been proposed in the literature, e.g.,
\citep{zinkevich2003online, kalai2005efficient, hazan2016introduction}.
Our proofs of the theoretical guarantees require
the no-regret online learning algorithm to be deterministic, i.e.,
$\theta _t$ to be a deterministic
rule of previous observed iterates $\theta _1, \ldots, \theta _{t - 1}$ and
loss functions $\ell _1, \ldots, \ell _{t - 1}$, for all $t \in [m]$.
Online gradient descent~\cite{zinkevich2003online} is an example of such an algorithm.

In Theorem~\ref{th:no_regret_pop}, we prove no-regret guarantees for the
case where the no-regret online
algorithm is presented with explicit expectations for the loss incurred by a
beam search policy.
In Theorem~\ref{th:lemma_cost_ub}, we upper bound the expected cost
incurred by a beam search policy as a function of its expected loss.
This result holds in cases where, at each beam,
the surrogate loss is an upper bound on the expected cost increase at that beam.
In Theorem~\ref{th:whp_bound}, we use Azuma-Hoeffding to
prove no-regret high probability bounds for the case where we only have access to
empirical expectations of the loss incurred by a policy,
rather than explicit expectations.
In Theorem~\ref{th:stopreset}, we extend Theorem~\ref{th:whp_bound} for
the case where the data
collection policy is different from the policy that we are evaluating.
These results allow us to give regret guarantees that depend
on how frequently is the data collection policy different from the
policy that we are evaluating.

In this section we simply state the results of the theorems alongside some discussion.
All proofs are presented in detail in Appendix~\ref{sec:noregret}.
Our analysis closely follows
that of DAgger~\cite{ross2011reduction},
although the results need to be interpreted in the beam search setting.
Our regret guarantees for beam-aware algorithms with different data collection
strategies are novel.

\subsection{No-Regret Guarantees with Explicit Expectations}

The sequence of functions $\ell _1, \ldots, \ell _m$ can be chosen in a way that
applying a no-regret online learning algorithm to generate the sequence of
policies $\theta _1, \ldots, \theta _m$ leads
to no-regret guarantees for
the performance of the mixture of $\theta _1, \ldots, \theta _m$.
The adversary presents the no-regret online learning algorithm
with $\ell _t = \ell(\cdot, \theta _t)$ at time $t \in [m]$.
The adversary is able to play $\ell(\cdot, \theta _t)$ because
it can anticipate $\theta _t$,
as the adversary knows the deterministic rule used by the no-regret online learning
algorithm to pick iterates.
Paraphrasing Theorem~\ref{th:no_regret_pop}, on the distribution of trajectories
induced by the
the uniform stochastic mixture of $\theta _1, \ldots, \theta _m$, the best policy
in $\Theta$ for this distribution performs as well (in the limit) as the
uniform mixture of $\theta _1, \ldots, \theta _m$.

\begin{theorem}
Let $\ell(\theta, \theta') =
  \Ex _{(x, y) \sim \mathcal D}
    \Ex _{b _{1:h} \sim \pi(\cdot, \theta')}
            \left( \sum _{i = 1} ^{h - 1} \ell(\theta, b _i) \right)$.
If the sequence $\theta _1, \ldots, \theta _m$ is chosen by a deterministic no-regret
online learning algorithm, we have
$\frac 1 m \sum _{t = 1} ^m \ell(\theta_t, \theta _t)
    - \min _{\theta \in \Theta} \frac 1 m
      \sum _{t = 1} ^m \ell(\theta, \theta _t) = \gamma _m$,
    where $\gamma _m$ goes to zero when $m$ goes to infinity.
  \label{th:no_regret_pop}
\end{theorem}

Furthermore, if for all $(x, y) \in \mathcal X \times \mathcal Y$
the surrogate loss $\ell(\cdot, b) : \Theta \to \mathbb R$ is
an upper bound on the expected cost increase
$\Ex _{b’ \sim \pi(b, \cdot)} c(b, b’) : \Theta \to \mathbb R$
for all $b \in V _k$, we can transform the surrogate loss no-regret guarantees
into performance guarantees in terms of $c : \mathcal Y \to \mathbb R$.
Theorem~\ref{th:lemma_cost_ub} tells us that if the best policy
along the trajectories induced
by the mixture of $\theta _1, \ldots, \theta _m$ in $\Theta$ incurs small
surrogate loss, then the expected cost resulting from labeling examples
$(x, y) \in \mathcal X \times \mathcal Y$ sampled from $\mathcal D$
with the uniform mixture of $\theta _1, \ldots, \theta _m$ is also small.
It is possible to transform the results about the uniform mixture
of $\theta _1, \ldots, \theta _m$
on results about the best policy among $\theta _1, \ldots, \theta _m$, e.g.,
following the arguments of \citet{cesa2004generalization},
but for brevity we do not present them in this paper.
Proofs of Theorem~\ref{th:no_regret_pop} and Theorem
\ref{th:lemma_cost_ub} are in Appendix~\ref{sec:noregret_explicit_expectations}

\begin{theorem}
  Let all the conditions in Definition~\ref{th:no_regret_alg} be satisfied.
  Additionally, let $c(\theta) =  c ^*(b _1) +
  \Ex _{(x, y) \sim \mathcal D}
    \Ex _{b _{1:h} \sim \pi(\cdot, \theta)}
      \left( \sum _{i = 1} ^{h - 1} c(b_i, b _{i + 1}) \right) =
    \Ex _{(x, y) \sim \mathcal D}
      \Ex _{b _{1:h} \sim \pi(\cdot, \theta)}
        c ^*(b _h)$.
  Let $\ell(\cdot, b) : \Theta \to \mathbb R$ be an upper bound on
  $\Ex _{b' \sim \pi(b, \cdot)} c(b, b') : \Theta \to \mathbb R$,
  for all $b \in V _k$.
  Then, $\frac 1 m \sum _{t = 1} ^m c(\theta _t) \leq
    \Ex _{(x, y) \sim \mathcal D} c ^*(b _1) +
      \min _{\theta \in \Theta} \frac 1 m
        \sum _{t = 1} ^m \ell(\theta, \theta _t) + \gamma _m$,
    where $\gamma _m$ goes to zero as $m$ goes to infinity.
    \label{th:lemma_cost_ub}
\end{theorem}

\subsection{Finite Sample Analysis}

Theorem~\ref{th:no_regret_pop} and Theorem~\ref{th:lemma_cost_ub}
are for the case where the adversary
presents explicit expectations, i.e., the loss function at time $t \in [m]$ is
$\ell _t(\cdot) = \Ex _{(x, y) \sim \mathcal D}
  \Ex _{b _{1:h} \sim \pi(\cdot, \theta _t) }
    \left( \sum _{i = 1} ^{h - 1} \ell(\cdot, b _i) \right)$.
We most likely only have access to a sample estimator
$\hat \ell(\cdot, \theta _t): \Theta \to \mathbb R$ of the true expectation:
we first sample an example $(x _t, y _t) \sim \mathcal D$, sample
a trajectory $b _{1:h}$ according to $\pi(\cdot, \theta _t)$,
and obtain $\hat \ell(\cdot, \theta _t) = \sum _{i = 1} ^{h - 1} \ell(\cdot, b _i)$.
We prove high probability no-regret guarantees for this case.
Theorem~\ref{th:whp_bound} tells us that the population surrogate loss
of the mixture of policies $\theta _1, \ldots, \theta _m$
is, with high probability, not much larger than its
empirical surrogate loss.
Combining this result with Theorem~\ref{th:no_regret_pop}
and Theorem~\ref{th:lemma_cost_ub} allows us to give finite sample
high probability results for the performance of the mixture of policies
$\theta _1, \ldots, \theta _m$.
The proof of Theorem~\ref{th:whp_bound} is found in Appendix~\ref{sec:finite_sample}.

\begin{theorem}
  Let $\hat \ell(\cdot, \theta') = \sum _{i = 1} ^{h - 1} \ell(\cdot, b _i)$
  which is generated by sampling $(x, y)$ from $\mathcal D$ (which induces
  the corresponding beam search space $G _k$ and cost functions),
  and sampling a beam trajectory using $\pi(\cdot, \theta')$.
  Let $| \sum _{i = 1} ^{h - 1} \ell(\theta, b _i) | \leq u$
  for a constant $u \in \mathbb R$, for all
  $(x, y) \in \mathcal X \times \mathcal Y$, beam trajectories $b_{1:h}$,
  and $\theta \in \Theta$.
  Let the iterates be chosen by a no-regret online learning algorithm,
  based on the sequence of losses
  $\ell _t = \hat \ell(\cdot, \theta _t) : \Theta \to \mathbb R$,
  for $t \in [m]$, then we have
  $\mathbb P
    \left(
      \frac 1 m \sum _{t = 1} ^m \ell(\theta _t, \theta _t) \leq
        \frac 1 m \sum _{t = 1} ^m \hat \ell(\theta _t, \theta _t) + \eta(\delta, m)
    \right) \geq 1 - \delta$,
    where $\delta \in (0, 1]$ and
    $\eta(\delta, m) = u \sqrt{2 \log(1 / \delta) / m}$.
  \label{th:whp_bound}
\end{theorem}

\subsection{Finite Sample Analysis for Arbitrary Data Collection Policies}
\label{sec:main_arbitrary_dc_policies}

All the results stated so far are for the continue data collection strategy
where, at time $t \in [m]$,
the whole trajectory $b _{1:h}$ is collected using the
current policy $\pi(\cdot, \theta _t)$.
Stop and reset data collection strategies do not necessarily collect the
full trajectory under $\pi(\cdot, \theta _t)$.
If the data collection policy $\pi' : V _k \to \Delta(V _k)$
is other than the learned policy, the
analysis can be adapted by accounting
for the difference in distribution of trajectories induced by
the learned policy and the data collection policy.
The insight is that $\sum _{i = 1} ^{h - 1} \ell(\theta, b _i)$ only
depends on $b _{1:h - 1}$, so if no cost increases occur
in this portion of the trajectory, we are effectively sampling the
trajectory using $\pi(\cdot, \theta)$ when using the stop and
reset data collection strategies.

Prior work presented only perceptron-style results for these settings
\cite{collins_incremental_2004,daume2005learning}---we are the first
to present regret guarantees.
Our guarantee depends on the probability with which $b _{1:h - 1}$
is collected solely with $\pi(\cdot, \theta)$.
We state the finite sample analysis result for the case where these
probabilities are not known explicitly, but we are able to estimate them.
The proof of Theorem~\ref{th:stopreset} is found in Appendix~\ref{sec:arbitrary_dc_policies}.

\begin{theorem}
  Let $\pi _t : \V _k \to \Delta(V_k)$ be the data collection
  policy for example $t \in [m]$, which uses either the stop or reset
  data collection strategies.
  Let $\hat \alpha(\theta _t)$ be the empirical estimate of
  the probability of $\pi(\cdot, \theta _t)$ incurring at least one cost increase
  up to time $h - 1$.
  Then,
  \begin{align*}
    &\mathbb P\left(
      \frac 1 m
        \sum _{t = 1} ^m
          \ell(\theta _t, \theta _t) \leq
          \frac 1 m
            \sum _{t = 1} ^m
              \hat \ell(\theta _t, \pi _t)
              +  u \left(
                  1 - \frac 1 m \sum _{t = 1} ^m \hat \alpha(\theta _t)
                \right)
              + 2 \eta(\delta, m)
    \right)
      \geq 1 - \delta,
  \end{align*}
  where $\delta \in (0, 1]$ and
  $\eta(\delta, m) = u \sqrt{2 \log(1 / \delta) / m}$.
  \label{th:stopreset}
\end{theorem}
If the probability of stopping or resetting goes to zero as $m$ goes to
infinity, then the term captures the discrepancy between the distributions of
induced by $\pi(\cdot, \theta _t)$ and $\pi_t$ vanishes, and we recover
a guarantee similar to Theorem~\ref{th:whp_bound}.
If the probability of stopping or resetting does not go completely to zero,
it is still possible to provide regret guarantees for the
performance of this algorithm but now with a term that does not vanish
with increasing $m$.
These regret guarantees for the different data collection strategies are novel.

\section{Conclusion}
\label{sec:conclusion}

We propose a framework for learning beam search policies using imitation learning.
We provide regret guarantees for both new and existing algorithms for learning
beam search policies.
One of the main contributions is formulating learning beam search policies in the
learning to search framework.
Policies for beam search are induced via a scoring function.
The intuition is that the best neighbors in
a beam should be scored sufficiently high, allowing them to be kept in the beam
when transitioning using these scores.
Based on this insight, we motivate different surrogate loss functions for
learning scoring functions.
We recover existing algorithms in the literature through
specific choices for the loss function and data collection strategy.
Our work is the first to provide a beam-aware algorithm with
no-regret guarantees.

\subsubsection*{Acknowledgments}

The authors would like to thank
Ruslan Salakhutdinov, Akshay Krishnamurthy, Wen Sun, Christoph Dann, and Kin Olivares for
helpful discussions and detailed reviews.

{
\small
\bibliography{beam_learn}

\begin{thebibliography}{10}

\bibitem{sutskever2014sequence}
Ilya Sutskever, Oriol Vinyals, and Quoc Le.
\newblock Sequence to sequence learning with neural networks.
\newblock {\em NIPS}, 2014.

\bibitem{graves2013speech}
Alex Graves, Abdel-rahman Mohamed, and Geoffrey Hinton.
\newblock Speech recognition with deep recurrent neural networks.
\newblock {\em ICASSP}, 2013.

\bibitem{vinyals2015show}
Oriol Vinyals, Alexander Toshev, Samy Bengio, and Dumitru Erhan.
\newblock Show and tell: A neural image caption generator.
\newblock {\em CVPR}, 2015.

\bibitem{weiss2015structured}
David Weiss, Chris Alberti, Michael Collins, and Slav Petrov.
\newblock Structured training for neural network transition-based parsing.
\newblock {\em ACL}, 2015.

\bibitem{ross2011reduction}
St{\'e}phane Ross, Geoffrey Gordon, and Drew Bagnell.
\newblock A reduction of imitation learning and structured prediction to
  no-regret online learning.
\newblock {\em AISTATS}, 2011.

\bibitem{collins_incremental_2004}
Michael Collins and Brian Roark.
\newblock Incremental parsing with the perceptron algorithm.
\newblock {\em ACL}, 2004.

\bibitem{daume2005learning}
Hal Daum{\'e} and Daniel Marcu.
\newblock Learning as search optimization: Approximate large margin methods for
  structured prediction.
\newblock {\em ICML}, 2005.

\bibitem{huang2012structured}
Liang Huang, Suphan Fayong, and Yang Guo.
\newblock Structured perceptron with inexact search.
\newblock {\em NAACL}, 2012.

\bibitem{andor2016globally}
Daniel Andor, Chris Alberti, David Weiss, Aliaksei Severyn, Alessandro Presta,
  Kuzman Ganchev, Slav Petrov, and Michael Collins.
\newblock Globally normalized transition-based neural networks.
\newblock {\em ACL}, 2016.

\bibitem{xu2007learning}
Yuehua Xu and Alan Fern.
\newblock On learning linear ranking functions for beam search.
\newblock {\em ICML}, 2007.

\bibitem{wiseman2016sequence}
Sam Wiseman and Alexander Rush.
\newblock Sequence-to-sequence learning as beam-search optimization.
\newblock {\em ACL}, 2016.

\bibitem{goyal2017continuous}
Kartik Goyal, Graham Neubig, Chris Dyer, and Taylor Berg-Kirkpatrick.
\newblock A continuous relaxation of beam search for end-to-end training of
  neural sequence models.
\newblock {\em AAAI}, 2018.

\bibitem{daume2009search-based}
Hal Daum{\'e}, John Langford, and Daniel Marcu.
\newblock Search-based structured prediction.
\newblock {\em Machine learning}, 2009.

\bibitem{bengio2015scheduled}
Samy Bengio, Oriol Vinyals, Navdeep Jaitly, and Noam Shazeer.
\newblock Scheduled sampling for sequence prediction with recurrent neural
  networks.
\newblock {\em NIPS}, 2015.

\bibitem{ross2014reinforcement}
St{\'e}phane Ross and Andrew Bagnell.
\newblock Reinforcement and imitation learning via interactive no-regret
  learning.
\newblock {\em arXiv preprint arXiv:1406.5979}, 2014.

\bibitem{chang2015learning}
Kai-Wei Chang, Akshay Krishnamurthy, Alekh Agarwal, Hal Daum{\'e}, and John
  Langford.
\newblock Learning to search better than your teacher.
\newblock {\em ICML}, 2015.

\bibitem{beygelzimer2008machine}
Alina Beygelzimer, John Langford, and Bianca Zadrozny.
\newblock Machine learning techniques—reductions between prediction quality
  metrics.
\newblock {\em Performance Modeling and Engineering}, 2008.

\bibitem{kingma2015adam}
Diederik Kingma and Jimmy Ba.
\newblock Adam: A method for stochastic optimization.
\newblock {\em ICLR}, 2015.

\bibitem{papineni2002bleu}
Kishore Papineni, Salim Roukos, Todd Ward, and Wei-Jing Zhu.
\newblock Bleu: a method for automatic evaluation of machine translation.
\newblock {\em ACL}, 2002.

\bibitem{zinkevich2003online}
Martin Zinkevich.
\newblock Online convex programming and generalized infinitesimal gradient
  ascent.
\newblock {\em ICML}, 2003.

\bibitem{kalai2005efficient}
Adam Kalai and Santosh Vempala.
\newblock Efficient algorithms for online decision problems.
\newblock {\em Journal of Computer and System Sciences}, 2005.

\bibitem{hazan2016introduction}
Elad Hazan.
\newblock Introduction to online convex optimization.
\newblock {\em Foundations and Trends{\textregistered} in Optimization}, 2016.

\bibitem{cesa2004generalization}
Nicolo Cesa-Bianchi, Alex Conconi, and Claudio Gentile.
\newblock On the generalization ability of on-line learning algorithms.
\newblock {\em IEEE Transactions on Information Theory}, 2004.

\bibitem{taskar2003max-margin}
Ben Taskar, Carlos Guestrin, and Daphne Koller.
\newblock Max-margin {Markov} networks.
\newblock {\em NIPS}, 2003.

\bibitem{gimpel2010softmax-margin}
Kevin Gimpel and Noah Smith.
\newblock Softmax-margin {CRF}s: Training log-linear models with cost
  functions.
\newblock In {\em ACL}, 2010.

\end{thebibliography}
\bibliographystyle{unsrt}
}

\appendix
\onecolumn

\section{Conversion to Tree-Structured Search Spaces}
\label{sec:ConversionToTreeStructured}

We define a search space as an arbitrary
finite directed graph $G = (V, E)$,
where $V$ is the set of nodes and $E \subset \V\times V$
is the set of directed edges.
Every directed graph $G = (V, E)$ has associated a tree-structured
directed graph $G_p = (V_p, E_p)$ encoding all possible paths through $G$.
An important reason to do this transformation is that, in practice,
policies often incorporate history features,
so they are functions of the whole path leading to a
node in $G$, rather than just a single node in $G$.
A policy becomes a function of single nodes of $G _p$.
If $G$ is tree-structured, $G _p$ is isomorphic to $G$, i.e., they
are the same search space.

The set of terminal nodes $T _p$ contains all
paths from the
initial node $\sinit \in V$ to terminal nodes $v \in T$.
For $v \in V_p$,
we denote
the length of the sequence encoding a
path by $|v|$.
The length of a path $v \in V _p$ is $|v| - 1$.
We write $v _i$ for the $i$-th element of
a path $v \in V _p$.
For all $v \in V$,
$v _i \in V$ for all $i \in [|v|]$ and $v _1 = \sinit$.
The sets
$N _{p, v}, R _{p, v}, T _{p, v}$
for $v \in V _p$
are defined analogously to the
sets $N _v, R _v, T _v$ for
$v \in V$.
For a path $v \in V _p$, $v' \in N _{p, v}$
if $v' _{1: |v| } = v$, $|v| = |v'| - 1$, and
$v' _{ |v'|} \in N _{v' _{ |v|} }$, i.e.,
a path $v' \in V _p$  neighbors $v \in V _p$ if
it can be written as $v$ followed by an additional node in $N _{v _{ |v| }}$.
For $v \in V _p$,
$v' \in R _{p, v}$ if $v$ is a prefix of $v'$
and $v' \in T _{p, v}$ if
$v$ is a prefix of $v'$ and $v' _{ |v'| } \in T$.
As $G _p$ is tree-structured, we can define the depth $d _v$
of a path $v \in V _p$
as its length, i.e., $d _v = |v| - 1$.
If path $v \in V _p$, then prefix $v _{1:i} \in V _p$,
for all $i \in [|v|]$,
i.e., path prefixes are themselves paths.

Tree-structured search spaces are
common in practice.
They often occur in
write-only search spaces,
where once an action is taken,
its effects are irreversible.
Typical search spaces for
sequence tagging and machine translation
are tree-structured:
given a sequence to tag or translate,
at each step we commit to a token
and never get to change it.
When the search space $G$ is not naturally seen as being
tree-structured, the construction described makes it natural to
work with an equivalent tree-structured
search space of paths $G _p$.

If $G$ has cycles, $G _p$ would be infinite.
Infinite cycling in $G _p$ can be prevented by, for example,
introducing a maximum path length or a maximum number of times that
any given node $v \in V$ can be visited.
In this paper, we also assumed that all nodes in
$T _p$ have distance $h$ to the root.
It is possible to transform $G _p$ into
a new tree-structured graph $G' _p$ by padding shorter paths to length $h$.
Let $h$ be the maximum distance of any terminal in $T _p$ to the root.
For each terminal node $v \in T _p$ with distance $d _v < h$ to the root,
we extend the path to $v$ by
appending a linear chain of $h - d _v$ additional nodes.
Node $v$ is no longer a terminal node in $G' _p$, and
all the nodes in $G' _p$ that resulted from extending
the path are identified with $v$.

\section{Convex Upper Bound Surrogate for Expected Beam Transition Cost}
\label{sec:upper_bound}

In this appendix, we design a convex upper bound surrogate
loss $\ell(\cdot, b) : \Theta \to \mathbb R$
for the expected beam transition cost
$\Ex _{b' \sim \pi(b, \cdot) } c(b, b') : \Theta \to \mathbb R$.
Let $A _b = \{ v _1, \ldots, v _n \}$ be an
arbitrary ordering of the neighbors of
$b$, with corresponding costs $c _1, \ldots, c _n$,
with $c _i = c^*(v _i)$ for all $i \in [n]$.
Let $s _1, \ldots, s _n$ be the corresponding scores, with
$s _i = s(v _i, \theta)$ for all $i \in [n]$.
Let $\sigma ^*: [n] \to [n]$
and $\hat \sigma : [n] \to [n]$ be the unique permutations
such that $c _{\sigma ^*(1)} \leq \ldots \leq c _{\sigma ^*(n)}$
and $s _{\hat \sigma (1)} \geq \ldots \geq s _{\hat \sigma (n)}$, respectively,
with ties broken according to the total order on $V$.
We have $c^*(b) = c _{ \sigma ^* (1) }$.
Let $k \in \mathbb N$ be the maximum beam capacity.
Let $b'$ be the beam induced by the scores $s _1, \ldots, s _n$, i.e.,
$b' = \{ v _{ \hat \sigma(1)}, \ldots, v _{ \hat \sigma(k')} \}$,
with $k' = \min(k, n)$ and ties broken according to the total order.

Consider the \emph{upper bound} loss function
(repeated here from Equation \eqref{eq:cvx_upper_bound})
\begin{align}
    \ell(s, c) =
    \max
      \left(
        0, \delta _{k + 1}, \ldots, \delta _n
    \right) ,
    \label{eq:cvx_upper_bound_repeat}
\end{align}
where $\delta _j =
( c _{ \sigma ^*(j) } - c _{ \sigma ^*(1)} )
( s _{\sigma ^*(j)} - s _{\sigma ^*(1)} + 1 )$
for $j \in \{k + 1, \ldots,  n\}$.

This loss function is lower bounded by zero,
so we only need to show that it upper bounds $c(b, b')$ when
there is a cost increase, i.e., when $c(b, b') > 0$.
A cost increase $c(b, b') > 0$ implies that the best element
$v _{\sigma ^*(1)}$ fell off the beam, meaning that
$b' = \{ v _{ \hat \sigma (1) }, \ldots, v _{ \hat \sigma (k) } \} \neq
  \{ v _{ \sigma ^* (1) }, \ldots, v _{ \sigma ^* (k) } \}$, and therefore
$b' \cap \{ v _{ \sigma ^* (k + 1) }, \ldots, v _{ \sigma ^* (n) } \}
  \neq \emptyset$.
Let $v _{\sigma ^*(j)} \in b' \cap \{ v _{ \sigma ^* (k + 1) }, \ldots, v _{ \sigma ^* (n) } \}$,
then $s _{\sigma ^*(j)} \geq s _{ \sigma ^*(1) }$ and
$c(b, b') \leq c _{\sigma ^*(j)} - c _{ \sigma ^* (1) }$,
with $j \in \{ k + 1, \ldots, n \}$.
We have
\begin{align*}
\max
  \left(
    0, \delta_{k + 1}, \ldots, \delta_n
  \right)
  &\geq
  \delta _j \\
  &=
  ( c _{ \sigma ^*(j) } - c _{ \sigma ^*(1)} )
  ( s _{\sigma ^*(j)} - s _{\sigma ^*(1)} + 1 ) \\
  &\geq
    c _{\sigma ^*(j)} - c _{ \sigma ^* (1) } \\
  &\geq c(b, b'),
\end{align*}
proving the upper bound property of the loss
in Equation~\eqref{eq:cvx_upper_bound_repeat}.

This loss is the maximum of a finite
number of affine functions of the scores, and therefore convex
with respect to the score vector $s \in \mathbb R ^n$.
The resulting optimization problem is convex
with respect to the parameters of the scoring function if, for example,
the scoring function is linear with respect to the parameters $\theta \in \Theta$,
i.e., $s(v, \theta) = \theta ^T \phi(v, x)$,
where $\phi : V \times \mathcal X \to \mathbb R ^p$ is a fixed feature
function of the state.
If $A_b$ has no more than $k$ elements, this surrogate loss
is identically zero, i.e., for $k \geq n$, $\ell(s, c) = 0$, for all
$s \in \mathbb R ^n$ and $c \in \mathbb R ^n$.
If $k = 1$, we recover a greedy decoding algorithm and the
loss in Equation~\eqref{eq:cvx_upper_bound_repeat} becomes a weighted hinge loss.

\section{Convexity Considerations for Surrogate Loss Functions}
\label{sec:convex_cosiderations}

It is common in the literature to update the parameters only
when a cost increase occurs
\cite{xu2007learning, huang2012structured, andor2016globally}.
We show that the resulting loss surrogate functions are, in general, non-convex
in the scores.

The following loss is an upper bound on the beam transition loss
$c : \Ek \to \mathbb R$, but is non-convex in the scores:
\begin{align}
  \ell(s, c) =
    (c _{\hat \sigma (k)} - c _{\sigma ^*(1) })
    \max (0, s _{ \hat \sigma (k) } - s _{ \sigma ^* (1) } + 1) .
  \label{eq:wiseman}
\end{align}
The upper bound property for this loss is easy to verify:
if $s \in \mathbb R ^n$ at $b \in \Vk$
induces $b' \in \Vk$ with $c(b, b') > 0$, then
$s _{ \hat \sigma (k) } \geq s _{ \sigma ^* (1) }$ and
$c _{ \hat \sigma (k) } > c _{ \sigma ^* (1) }$,
leading to
\begin{align*}
  (c _{\hat \sigma (k)} - c _{\sigma ^*(1) })
    \max (0, s _{ \hat \sigma (k) } - s _{ \sigma ^* (1) } + 1)
    &\geq c _{\hat \sigma (k)} - c _{\sigma ^*(1) } \\
    &\geq c(b, b'),
\end{align*}
as $v _{\hat \sigma (k)} \in b'$.
This loss is used in \citet{wiseman2016sequence}.
The same reasoning holds when substituting $k$ in
Equation~\eqref{eq:wiseman} by any $i \in [k]$.

We now show
that two aspects commonly present in
the beam-aware literature lead to non-convexity
of the surrogate losses.
The first aspect is
updating the parameters only when
there is a cost increase.
This amounts to defining a new loss function
$\ell' : \mathbb R ^n \times \mathbb R ^n \to \mathbb R$ from
$\ell : \mathbb R ^n \times \mathbb R ^n \to \mathbb R$ of the form
\begin{align*}
  \ell'(s, c) = \ell(s, c) \mathbbm{1}[ c(b, b') > 0 ],
\end{align*}
where $b'$ is induced by $s \in \mathbb R ^n$.
The second aspect that leads to non-convexity is indexing
the score vector $s \in \mathbb R ^n$ or
cost vector $c \in \mathbb R ^n$ with a function of the parameters, e.g.,
permutation $\hat \sigma : [n] \to [n]$
depends on the scores $s \in \mathbb R ^n$ and therefore,
on the parameters $\theta \in \Theta$.
We show non-convexity with respect to the scores
through two simple counter examples.

For the first aspect,
let $k = 2$ and $n = 3$, with
$v _1, v _2, v _3$ having costs $c _1 = 0, c _2 = 1, c _3 = 1$.
Any beam that keeps $v _1$ has no cost increase.
Consider the scores
$s _1 = 1, s _2 = 10, s _3 = 0$
and
$s' _1 = 1, s' _2 = 0, s' _3 = 10$.
Both $s$ and $s'$ lead to no cost increase, as both score vectors
keep $v _1$ in the beam.
For $\ell' : \mathbb R ^n \times \mathbb R ^n \to \mathbb R$ to be convex
in the scores,
we must have
$\ell'(\alpha s + (1 - \alpha) s', c)
  \leq \alpha \ell'(s, c) + (1 - \alpha) \ell'(s', c)$, for all $\alpha \in [0, 1]$.
As both $s$ and $s'$ lead to no cost increase, we have
$\ell'(s, c) = \ell'(s', c) = 0$, yielding the
following necessary condition for convexity:
$\ell(\alpha s + (1 - \alpha) s', c) \leq 0$ for all $\alpha \in [0, 1]$.
For $\alpha = 0.5$, we have
$\overline{s} _1 = 1, \overline{s} _2 = 5, \overline{s} _3 = 5$,
which leads to a cost increase, and therefore to
loss $\ell'(\overline{s}, c) > 0$,
implying that $\ell' : \mathbb R^n \times \mathbb R^n \to \mathbb R$
is non-convex in the scores.

For the second aspect,
consider the loss in
Equation~\eqref{eq:wiseman}.
Ignore the multiplicative term involving the costs
and consider only the hinge part
$\max(0, s _{ \hat \sigma (k) } - s _{ \sigma ^* (k) } + 1 )$.
Let $k = 2$ and $n = 3$.
Consider that the elements $v _1, v _2, v _3$ are sorted in
increasing order of cost; let $s _1 = 2, s _2 = 1, s _3 = 0$,
and $s' _1 = 2, s' _2 = 4, s' _3 = 0$.
In both cases, the hinge part of loss in Equation~\eqref{eq:wiseman}
is zero, but if we take a
convex combination of the scores with $\alpha = 0.5$,
we get $\overline s _1 = 2, \overline s _2 = 2.5, \overline s _3 = 0$,
for which the surrogate loss is nonzero (assuming that the costs of
$v _1, v _2, v _3$ are unique).

\section{Additional Loss Functions}
\label{sec:additional_losses}

We present additional loss functions that were omitted
in Section~\ref{sec:surrogate_losses} and discuss their connections
to previous work.
\paragraph{cost sensitive margin (beam)}
  Prefers the lowest cost element to be scored higher than best runner-up
  in the beam by a cost-weighted margin.
  With unbounded beam capacity, we recover the
  structured max-margin loss of \citet{taskar2003max-margin} for M$^3$Ns.
  \begin{align}
    \ell(s, c) =
      - s _{ \sigma ^*(1) } +
      \max_{i \in \{1,\ldots,k\}} \left(
          c _{\hat \sigma (i) } + s _{ \hat \sigma (i)} \right)
  \end{align}
\paragraph{softmax margin (beam)}
  Log loss that can be understood as
  smoothing the $\max$ in {\it cost sensitive margin (beam)}.
  With unbounded beam capacity, we recover the
  softmax-margin loss of \citet{gimpel2010softmax-margin} for CRFs.
  \begin{align}
          \ell(s, c)
          &= - s _{ \sigma ^*(1) } +
              \log \left(
                \sum _{ i = 1} ^k \exp \left( c_{\hat \sigma (i)} + s _{\hat \sigma (i)} \right)
              \right)
  \end{align}
  \paragraph{weighted pairs (all)}
  Reduces the problem of producing the correct ranking over the neighbors to
  $n(n - 1) / 2$ weighted binary classification problems.
  Hinge terms for pairs with the same cost cancel, effectively expressing
  that we are indifferent to the relative order of the elements of the pair.
  \begin{align}
    \ell(s, c) =
    \sum _{i = 1} ^n
      \sum _{j = i + 1} ^n
        \left( c _{ \sigma ^*(j) } - c _{ \sigma ^*(i)} \right)
        \max \left( 0,  s _{\sigma ^*(j)} - s _{\sigma ^*(i)} + 1 \right)
  \end{align}
  \paragraph{weighted pairs (bipartite)}
  Only weighted pairs between elements than ought to be included in the beam
  and those that ought to excluded from the beam.
  A similar loss has been proposed for
  bipartite ranking, where the
  goal is to order all positive examples before all negative examples
  \begin{align}
    \ell(s, c) =
    \sum _{i = 1} ^k
      \sum _{j = k + 1} ^n
        \left( c _{ \sigma ^*(j) } - c _{ \sigma ^*(i)} \right)
        \max \left( 0,  s _{\sigma ^*(j)} - s _{\sigma ^*(i)} + 1 \right)
  \end{align}
  \paragraph{weighted pairs (hybrid)}
  Similar to weighted pairs bipartite but we also include the pairs for
  the elements that ought to be included in the beam
  \begin{align}
    \ell(s, c) =
    \sum _{i = 1} ^k
      \sum _{j = i + 1} ^n
        \left( c _{ \sigma ^*(j) } - c _{ \sigma ^*(i)} \right)
        \max \left( 0,  s _{\sigma ^*(j)} - s _{\sigma ^*(i)} + 1 \right)
  \end{align}

  The \textit{weighted pairs (all)} loss provides many different variants as exemplified
  by \textit{weighted pairs (bipartite)} and \textit{weighted pairs (hybrid)}.
  We believe that exploring the ranking literature can lead to
  interesting insights on what losses to use for learning beam search policies
  in our framework.

\section{No-Regret Guarantees}
\label{sec:noregret}

This section presents analysis that leads to proofs of theorems
\ref{th:no_regret_pop},~\ref{th:lemma_cost_ub},~\ref{th:whp_bound}, and
\ref{th:stopreset}.
We analyze
\begin{align*}
  c(\theta) =
    \Ex _{(x, y) \sim \mathcal D}
      \Ex _{\hat y \sim \pi(\cdot, \theta)}
        c _{x, y} (\hat y) .
\end{align*}
The prediction cost $c _{x, y}(\hat y)$ is generated by sampling
a beam trajectory $b _{1:h}$ with
policy $\pi(\cdot, \theta)$.
The prediction $\hat y$ is extracted from $b _h$.
We have
\begin{align*}
  c(\theta) =
    \Ex _{(x, y) \sim \mathcal D}
      \Ex _{b _{1:h} \sim \pi(\cdot, \theta)}
        \left(
          c ^*(b _1) + \sum _{i = 1} ^{h - 1}
            c(b _i, b _{i + 1})
        \right).
\end{align*}
As $b _1$ depends only on $x \in \mathcal X$,
$c ^*(b _1)$ does not depend on the parameters $\theta$ and therefore
can be ignored for optimization purposes.
We analyze instead the surrogate
\begin{align}
  \ell (\theta, \theta') =
    \Ex _{(x, y) \sim \mathcal D}
      \Ex _{b _{1:h} \sim \pi(\cdot, \theta')}
        \left(
          \sum _{i = 1} ^{h - 1} \ell(\theta, b _i)
        \right),
  \label{eq:regret_full_exp}
\end{align}
where $\ell(\cdot, b) : \Theta \to \mathbb R$ is a surrogate for
$\Ex _{b' \sim \pi(b, \cdot)} c(b, b') : \Theta \to \mathbb R$.
See Section~\ref{sec:surrogate_losses} for extended discussion on the
motivation behind surrogate loss $\ell(\cdot, b)$.
It is convenient to
assume that the policy $\pi(\cdot, \theta') : V_k \to \Delta(V_k)$ used
to collect the beam trajectory $b_{1:h}$ can
be different than the policy $\pi(\cdot, \theta) : V_k \to \Delta(V_k)$
used to evaluate the surrogate losses at the visited beams.
The surrogate loss function $\ell : \Theta \times \Vk \to \mathbb R$
depends on the sampled example $(x, y) \in \mathcal X \times \mathcal Y$,
but we omit this dependency for conciseness.

\subsection{No-Regret Guarantees with Explicit Expectations}
\label{sec:noregret_explicit_expectations}

Here we present the proofs of Theorem~\ref{th:no_regret_pop} and Theorem~\ref{th:lemma_cost_ub}.
It is informative to consider the case where
we have access to both explicit expectations.
In this case, the no-regret algorithm is run on the sequence of losses
$\ell(\theta_1, \theta _1), \ldots, \ell(\theta_m, \theta _m)$
yielding average regret
\begin{align*}
  \gamma _m
    &= \frac 1 m
    \sum _{t = 1} ^m
        \ell(\theta _t, \theta _t) -
      \min _{\theta \in \Theta}
        \frac 1 m
      \sum _{t = 1} ^m \ell (\theta, \theta _t) .
\end{align*}
As the sequence $\theta _1, \ldots, \theta _m$ is generated by a no-regret
algorithm, the average regret goes to zero as $m$ goes to infinity.
This result tells us that
the uniform mixture obtained by sampling uniformly
at random one of $\theta _1, \ldots, \theta _m$ and acting according to it
for the full trajectory,
is competitive with the best policy in $\Theta$
along the same induced trajectories.
Note that
\begin{align*}
  \frac 1 T
    \sum _{t = 1} ^T \ell(\theta _t, \theta _t) -
    \min _{\theta \in \Theta}
      \frac 1 T
      \sum _{t = 1} ^T \ell (\theta, \theta _t)
  &= \Ex _{t \sim U(1, T)}
    \ell(\theta _t, \theta _t) -
  \min _{\theta \in \Theta}
    \Ex _{t \sim U(1, T)}
    \ell (\theta, \theta _t) ,
\end{align*}
where $U(1, T)$ denotes the uniform distribution over $[T]$.
Performance guarantees are obtained from the rearrangement
\begin{align*}
  \frac 1 m
    \sum _{t = 1} ^m \ell(\theta _t, \theta _t)
      &= \epsilon _m + \gamma _m,
\end{align*}
where
\begin{align*}
  \epsilon _m
    &= \min _{\theta \in \Theta}
      \frac 1 m
      \sum _{t = 1} ^m \ell (\theta, \theta _t) , \\
  \gamma _m
    &= \frac 1 m
      \sum _{t = 1} ^m \ell(\theta _t, \theta _t) -
      \min _{\theta \in \Theta}
        \frac 1 m
        \sum _{t = 1} ^m \ell (\theta, \theta _t) .
\end{align*}

Furthermore, if the surrogate loss
$\ell(\cdot, b) : \Theta \to \mathbb R$
upper bounds the expected
beam transition cost $\Ex _{b' \sim \pi(b, \cdot)} c(b, b') : \Theta \to \mathbb R$,
i.e., $\ell(\theta, b) \geq \Ex _{b' \sim \pi(b, \theta)} c(b, b')$ for
all $b \in V_k$ and all $\theta \in \Theta$, we have
\begin{align*}
  \mathbb E _{b _{1:h} \sim \pi(\cdot, \theta)}
    \left(
      \sum _{i = 1} ^{h - 1} c( b _i, b _{i + 1} )
    \right)
    &\leq
    \mathbb E _{b _{1:h} \sim \pi(\cdot, \theta)}
    \left(
      \sum _{i = 1} ^{h - 1} \ell(\theta, b _i)
    \right),
\end{align*}
and consequently,
\begin{align*}
  \frac 1 m \sum _{t = 1} ^m c(\theta _t)
    \leq \frac 1 m \sum _{t = 1} ^m \ell(\theta _t, \theta _t)
      + \Ex _{ (x, y) \sim \mathcal D } c^*(b _1),
\end{align*}
i.e, we are able to
use the expected surrogate loss incurred by the uniform mixture of
$\theta_1, \ldots, \theta _m$ to upper bound the
expected labeling cost resulting from labeling examples
$(x, y) \sim \mathcal D$ with the uniform mixture of
$\theta _1, \ldots, \theta _m$.

As the
sequence $\theta _1, \ldots, \theta _m$ is chosen by a no-regret algorithm,
$\gamma _m$ goes to zero as $m$ goes to infinity.
The term $\epsilon _m$ is harder to characterize as $m$ goes to infinity.
We are guaranteed that the uniform mixture of
$\theta _1, \ldots, \theta _m$ and, as result the best policy in
$\theta _1, \ldots, \theta _m$, is competitive with the best policy in hindsight
$\theta ^* _m$ $ \in
  \argmin _{\theta \in \Theta} 1 / m \sum _{t  = 1} ^m \ell(\theta, \theta _t)$.
For the performance guarantees to be interesting, it is necessary
for $\epsilon _m$ to remain small as $m$ goes to infinity, i.e.,
there must exist a policy in
$\Theta$ that performs well on the distribution of trajectories induced by
the uniform mixture of $\theta _1, \ldots, \theta _m$.
We think that this remark is often not adequately discussed in the literature.
Nonetheless, for expressive policy classes,
e.g., neural networks, it is reasonable to assume the existence of such a policy.

\subsection{Finite Sample Analysis}
\label{sec:finite_sample}

Next we provide a proof of Theorem~\ref{th:whp_bound}.
We typically do not have access to the explicit expectations in
Equation~\eqref{eq:regret_full_exp}.
What we do have access to is an estimator
\begin{align*}
  \hat \ell (\theta, \theta') = \sum _{i = 1} ^{h - 1} \ell(\theta, b _i),
\end{align*}
which is obtained by sampling an example $(x, y)$ from the
data generating distribution $\mathcal D$,
and executing policy $\pi(\cdot, \theta')$ to collect a trajectory
$b_{1:h}$.

Our no-regret algorithm is then run on the sequence of sampled losses,
yielding the sequence $\theta _1, \ldots, \theta _m$ and average regret
\begin{align*}
  \hat \gamma _m
  &= \frac 1 m
    \sum _{t = 1} ^m \hat \ell(\theta _t, \theta _t) -
    \min _{\theta \in \Theta}
      \frac 1 m
        \sum _{t = 1} ^m \hat \ell (\theta, \theta _t) .
\end{align*}
We show that the true population loss of the uniform mixture of
$\theta _1, \ldots, \theta _m$ is, with high probability,
not much larger than the empirical loss
observed on the sampled trajectories, i.e.,
\begin{align}
  \mathbb P\left(
    \frac 1 m
      \sum _{t = 1} ^m
        \ell(\theta _t, \theta _t) \leq
        \frac 1 m
          \sum _{t = 1} ^m
            \hat \ell(\theta _t, \theta _t)  + \eta(\delta, m)
  \right)
    \geq 1 - \delta,
    \label{eq:whp_result}
\end{align}
where $\delta \in (0, 1]$ is related to the probability of the statement, and
$\eta(\delta, m)$ depends only on $\delta$ and $m$.
Given this result, we are able to give performance guarantees for
the uniform mixture of $\theta _1, \ldots, \theta _m$
as
\begin{align}
  \mathbb P\left(
    \frac 1 m
      \sum _{t = 1} ^m \ell(\theta _t, \theta _t)
        \leq \hat \epsilon _m + \hat \gamma _m + \eta(\delta, m)
    \right) \geq 1 - \delta.
    \label{eq:whp_perf_guarantees}
  \end{align}
\begin{proof}
Define a function on beam trajectories.
Assume that we have $0 \leq \ell(\theta, b _{1:h}) \leq u$,
with $u \in \mathbb R$,
for all $(x, y) \in \mathcal X \times \mathcal Y$ and
for all beam trajectories $b _{1:h}$ through $G _k$,
i.e., $b _1 = \binit$,  $b _h \in T_k$,
$b _i \in V_k$ for all $i \in [n]$,
and $b _{i + 1} \in N _{b _i}$ for $i \in [h - 1]$.
As a result, $ 0 \leq \ell(\theta, \theta') \leq u$
and $0 \leq \hat \ell(\theta, \theta') \leq u$,
for all $\theta, \theta' \in \Theta$
and all $(x, y) \in \mathcal X \times \mathcal Y$.
In our case,
\begin{align}
  \ell(\theta, b_{1:h}) = \sum _{i = 1} ^{h - 1} \ell(\theta, b_i) .
\end{align}
Construct the martingale sequence
\begin{align}
  z _t = \sum _{j = 1} ^t
    \left(
      \ell(\theta _j, \theta _j) - \hat \ell( \theta _j, \theta _j)
    \right),
  \label{eq:first_martingale}
\end{align}
for $t \in [m]$.
It is simple to verify
that the sequence $z _1, \ldots, z _m$
is a martingale, i.e., that we have
$\Ex _{z _t | z _1, \ldots, z _{t - 1}} z _t = z _{t - 1}$ for all $t \in [m]$.
Furthermore, we have $| z _t - z _{t - 1} | \leq u$ for all $t \in [m]$,
where $z _0 = 0$.
The high probability result is obtained by applying the
Azuma-Hoeffding inequality
to the martingale sequence $z _t$, for $t \in \mathbb N$, which yields
\begin{align}
  \mathbb P\left(
    \frac 1 m
      \sum _{t = 1} ^m
        \ell(\theta _t, \theta _t) \leq
        \frac 1 m
          \sum _{t = 1} ^m
            \hat \ell(\theta _t, \theta _t)  + u \sqrt{
                \frac {2 \log (1 / \delta ) } {m} }
  \right)
    \geq 1 - \delta.
    \label{eq:whp_result_expanded}
\end{align}

Revisiting Equation~\eqref{eq:whp_perf_guarantees}, for fixed $\delta \in (0, 1]$,
as $m$ goes to infinity,
we have that both $\hat \gamma _m$ and $\eta(\delta, m)$ go to zero,
proving high probability no-regret guarantees for this setting.
\end{proof}

\subsection{Finite Sample Analysis for Arbitrary Data Collection Policies}
\label{sec:arbitrary_dc_policies}

Finally, in this section, we provide a proof of Theorem~\ref{th:stopreset}.
All the results stated so far are for the continue data collection strategy
where, at time $t \in [m]$,
the whole trajectory $b _{1:h}$ is collected using the
current policy $\pi(\cdot, \theta _t)$.
Stop and reset data collection strategies do not necessarily collect the
full trajectory under $\pi(\cdot, \theta _t)$.
If a transition $(b, b') \sim \pi(\cdot, \theta _t)$ leads to a cost increase,
then,
the stop data collection strategy stops collecting the trajectory at $b'$, and
the reset data collection strategy, the oracle policy $\pi^*(\cdot, c^*)$ is used to
sample the transition at $b$ instead.

In this section, we relate the expected loss of
$\pi(\cdot, \theta)$ on trajectories collected by
a different policy
$\pi'$ to the expected loss of $\pi(\cdot, \theta)$ on its own trajectories.
Consider the following auxiliary lemma:
\begin{lemma}
  Let $f : X \to \mathbb R$ be a function such that
  $f(x) \in [a, a + r]$, for $a, r \in \mathbb R$ and $r \geq 0$
  for all $x \in X$,
  that can be either discrete or continuous.
  Let $d ,d'$ be two probability distributions over $X$.
  We have
  \begin{align}
    | \mathbb E _{x \sim d} f(x) - \mathbb E _{x \sim d'} f(x)|
      &\leq r / 2  { || d - d' || } _1.
    \label{eq:lemma}
  \end{align}
  \label{lem:expectation}
\end{lemma}
\begin{proof}
  We prove the result for the case where $X$ is discrete, i.e., $d$ and $d'$
  are discrete probability distributions. The result for discrete distributions
  is sufficient for our purposes. Let $|X| = e$, with $e \in \mathbb N$, then
  $d, d' \in \mathbb R ^e$. We have
  \begin{align*}
    |\mathbb E _{x \sim d} f(x) - \mathbb E _{x \sim d'} f(x)|
      &= \left|\sum _{x \in X} d(x) f(x) - \sum _{x \in X} d'(x) f(x) \right| \\
      &= \left|\sum _{x \in X} d(x) ( f(x) - c )  - \sum _{x \in X} d'(x) (f(x) - c) \right| \\
      &= \left|(d - d') ^T ( f - c ) \right| \\
      &\leq || f - c || _\infty ||d - d'||_1,
  \end{align*}
  where $c$ is an arbitrary constant in $\mathbb R$ and $f \in \mathbb R ^e$
  is the vector representation of the function. In the second equality, we use
  $\sum _{x \in X} d(x) = \sum _{x \in X} d'(x) = 1$.
  In the third equality, we express the expectations as inner products and
  slightly abuse notation by denoting the coordinate-wise subtraction of $c$ from $f$ as $f - c$.
  In the
  final inequality, we use the generalized Cauchy–Schwarz inequality for the pair of dual norms
  $||\cdot || _1$ and $|| \cdot || _\infty$. The desired result is obtained by
  choosing $c = a + r/2$.
\end{proof}

Often,
$\pi' = (1 - \beta) \pi(\cdot, \theta) + \beta \pi ^*( \cdot, c^*)$
for $\beta \in [0, 1]$,
i.e., a probabilistic interpolation of the learned policy and the
oracle policy.
We do a more general analysis that will be useful to provide
regret guarantees for the stop and reset data collection strategies.
It is not necessarily the case that, for a roll-in policy
$\pi' : V _k \to \Delta(V _k)$, there exists $\theta' \in \Theta$
such that $\pi' = \pi (\cdot, \theta')$.
We modify the notation in Equation~\eqref{eq:regret_full_exp}
to capture this fact and write
\begin{align}
  \ell (\theta, \pi') =
  \Ex _{(x, y) \sim \mathcal D}
    \Ex _{b _{1:h} \sim \pi'}
      \left(
        \sum _{i = 1} ^{h - 1} \ell(\theta, b _i)
      \right) .
\end{align}

The roll-in policies $\pi' : V_k \to \Delta(V _k)$ that we
consider induce distributions over beam trajectories in $G _k$
that have a component where the beam trajectory up to $h - 1$ can be
thought as coming from $\pi(\cdot, \theta)$.
For a policy $\pi'$ that is somehow derived from the
learned policy $\pi(\cdot, \theta)$, we write
$d _{ \pi' } = \alpha(\theta, x, y) d _\theta +
  (1 - \alpha(\theta, x, y) ) q$, where $d _{\pi'}$
  is the distribution over trajectories induced by the roll-in policy $\pi'$,
  $d _\theta$ is the distribution over
  trajectories induced by the learned policy $\pi(\cdot, \theta)$,
  $q$ is the residual distribution over trajectories of the
  component that is not captured by $d _\theta$,
  and
  $\alpha(\theta, x, y)$ is the probability that the trajectory up to $b _{h - 1}$
  is drawn solely from $\pi(\cdot, \theta)$.
For example,
for the policy
$\pi' = (1 - \beta) \pi(\cdot, \theta) + \beta \pi ^*( \cdot, c^*)$,
we have $\alpha(\theta, x, y) = (1 - \beta) ^{h - 2}$, where
$\alpha(\theta, x, y)$ is independent of $\theta$ in this case.
In this example, $\pi'$, at each step of the trajectory of length $h$,
flips a biased coin and acts with
probability $1 - \beta$ according to $\pi(\cdot, \theta)$ and with
probability $\beta$ according to $\pi(\cdot, c^*)$.

\paragraph{Relating expectations}
We use Lemma~\ref{lem:expectation} to relate
$\mathbb E _{b _{1:h} \sim \pi(\cdot, \theta)}
  \left(
    \sum _{i = 1} ^{h - 1} \ell(\theta, b _i)
  \right)$
and
$\mathbb E _{b _{1:h} \sim \pi'}
  \left(
    \sum _{i = 1} ^{h - 1} \ell(\theta, b _i)
  \right)$.
We have
\begin{align*}
  {|| d _{\pi'} - d _{\theta} ||} _1
    &= {|| \alpha(\theta, x, y) d _{\theta} + (1 - \alpha(\theta, x, y) ) q  - d _{\theta} ||} _1 \\
    &= (1 - \alpha(\theta, x, y) ) {|| q  - d _{\theta} ||} _1 \\
    &\leq 2 (1 - \alpha(\theta, x, y) ),
\end{align*}
where we used that ${|| d_1 - d_2 ||} _1 \leq 2$ for any two distributions $d_1, d_2$.
Revisiting Equation~\eqref{eq:lemma}, we have
\begin{align*}
  \mathbb E _{b_{1:h} \sim \pi(\cdot, \theta)}
    \left(
      \sum _{i = 1} ^{h - 1} \ell(\theta, b _i)
    \right) \leq
    \mathbb E _{b_{1:h} \sim \pi'}
    \left(
      \sum _{i = 1} ^{h - 1} \ell(\theta, b _i)
    \right) + u (1 - \alpha(\theta, x, y)),
\end{align*}
and as a result
\begin{align}
  \ell(\theta, \theta)
    &=
    \mathbb E _{(x, y) \sim \mathcal D}
      \mathbb E _{b_{1:h} \sim \pi(\cdot, \theta)}
    \left(
      \sum _{i = 1} ^{h - 1} \ell(\theta, b _i)
    \right) \nonumber \\
    &\leq
    \mathbb E _{(x, y) \sim \mathcal D}
      \left(
      \mathbb E _{b_{1:h} \sim \pi'}
    \left(
      \sum _{i = 1} ^{h - 1} \ell(\theta, b _i)
    \right) + u (1 - \alpha(\theta, x, y)
    \right) \nonumber \\
    &=
    \ell(\theta, \pi') + u ( 1 - \alpha(\theta) ),
    \label{eq:rollin_difference}
\end{align}
where we defined
$\alpha(\theta) = \mathbb E _{(x, y) \sim \mathcal D} \alpha(\theta, x, y)$,
i.e., the probability of sampling the beam trajectory up to time $h - 1$ solely with
$\pi(\cdot, \theta)$, or equivalently, the probability of $\pi(\cdot, \theta)$
incurring no cost increases up to time $h - 1$.

\paragraph{Finite sample analysis with known schedules}
We now consider the finite sample analysis for the setting
considered in this section.
By arguments similar to those in Appendix~\ref{sec:finite_sample}, we have
\begin{align*}
  \mathbb P\left(
    \frac 1 m
      \sum _{t = 1} ^m
        \ell(\theta _t, \pi _t) \leq
        \frac 1 m
          \sum _{t = 1} ^m
            \hat \ell(\theta _t, \pi _t)  + u \sqrt{
                \frac {2 \log (1 / \delta ) } {m} }
  \right)
    \geq 1 - \delta,
\end{align*}
which, combining with Equation~\eqref{eq:rollin_difference} implies
\begin{align}
  \mathbb P\left(
    \frac 1 m
      \sum _{t = 1} ^m
        \ell(\theta _t, \theta _t) \leq
        \frac 1 m
          \sum _{t = 1} ^m
            \hat \ell(\theta _t, \pi _t)  + u \sqrt{
                \frac {2 \log (1 / \delta ) } {m} }
            + u \left(
                1 - \frac 1 m \sum _{t = 1} ^m \alpha(\theta _t)
              \right)
  \right)
    \geq 1 - \delta,
    \label{eq:whp_result_last}
\end{align}

Equation~\eqref{eq:whp_result_last} can be simplified
for roll-in policies
$\pi _t = (1 - \beta _t) \pi(\cdot, \theta) + \beta _t \pi ^*(\cdot, c ^*)$
with fixed interpolation schedules $\beta _t$, for $t \in \mathbb N$.
For example, for $\beta _1 = 1$  for $t \in [t _0]$,
for some $t _0 \in \mathbb N$, and $\beta _t = 0$ for $t > t _0$, we have
\begin{align}
  \mathbb P\left(
    \frac 1 m
      \sum _{t = 1} ^m
        \ell(\theta _t, \theta _t) \leq
        \frac 1 m
          \sum _{t = 1} ^m
            \hat \ell(\theta _t, \pi _t)  + u \sqrt{
                \frac {2 \log (1 / \delta ) } {m} }
            + u
                \min
                  \left(
                    1,
                    \frac {t _0} m
                  \right)
  \right)
    \geq 1 - \delta.
    \label{eq:whp_result_last_simplified}
\end{align}

\paragraph{Guarantees for the stop and reset data collection strategies}

The previous analysis allows us to provide regret guarantees
for the reset data collection strategy.
Steps in the trajectory are sampled using
the learned policy $\pi(\cdot, \theta)$ when they do not result
in cost increase, and sampled from $\pi^*(\cdot, c^*)$ otherwise, i.e.,
while sampling
a trajectory $b _1, \ldots, b _i$ with $\pi(\cdot, \theta)$, if a
cost increase would occur on the transition from $b _i$ to
$b' \sim \pi(b _i, \theta)$, then rather than transitioning to
$b _{i + 1} = b'$, we transition to $b _{i + 1} \sim \pi ^*(b _i, c^*)$,
and continue from $b _{i + 1}$ until a terminal beam $b _h \in T _k$ is reached.
In this case, $\alpha(\theta, x, y)$ is interpreted as the probability that
the trajectory $b _{1:{h - 1}}$ on the beam search $G _k$
induced by $x$ is sampled using only $\pi(\cdot, \theta)$,
i.e., no cost increases occur up to time $h - 1$.

We can use this fact along with the previous results
to obtain a regret statement for
both the explicit expectation
and the finite sample cases.
The main difficulty is that $\alpha(\theta, x, y)$ and $\alpha(\theta)$
are not known.
Again, the only way that we have access to them is through a sample estimate
$\hat \alpha(\theta)$.
We construct a martingale for this case involving both the randomness of
the loss function and the reset probability.

We can use this information along with Azuma-Hoeffding inequality to
give a joint concentration result.
The martingale sequence that we now construct is
\begin{align}
  z _t =
    \sum _{j = 1} ^t \left(
    \ell(\theta _j, \pi _j) -
        \hat \ell(\theta _j, \pi _j) +
      u \left(
        1 - \alpha (\theta _j)
      \right) -
        u \left(
          1 - \hat \alpha (\theta _j)
        \right)
    \right) ,
    \label{eq:second_martingale}
\end{align}
which now includes the random variables of the estimator of the
probability that we will reset at least once.
Note that $\hat \alpha (\theta)$ also depends on
  $x, y, b _{1:h}$, which we omit for simplicity.
Similarly to the martingale arguments in Equation~\eqref{eq:first_martingale},
Equation~\eqref{eq:second_martingale} defines a martingale.
In this case, we have $| z _t - z _{t - 1} | \leq 2u$ for all $t \in [m]$,
and $z _0 = 0$.
Applying Azuma-Hoeffding yields a result similar
to Equation~\eqref{eq:whp_result_last}, i.e.,
\begin{align}
  \mathbb P\left(
    \frac 1 m
      \sum _{t = 1} ^m
        \ell(\theta _t, \theta _t) \leq
        \frac 1 m
          \sum _{t = 1} ^m
            \hat \ell(\theta _t, \pi _t)  +
            2u \sqrt{
                \frac {2 \log (1 / \delta ) } {m} }
            + u \left(
                1 - \frac 1 m \sum _{t = 1} ^m \hat \alpha(\theta _t)
              \right)
  \right)
    \geq 1 - \delta,
  \label{eq:martingale_both}
\end{align}

Even if $1 / m \sum _{t = 1} ^m \hat \alpha (\theta _t)$ remains at
some nonzero quantity as $m$ goes to infinity, we can
still give a guarantee with respect to this reset probability.
Namely, if we observe that we are most of the time sampling the full
trajectory with the learned policy, then we guarantee that we are not
too far away from the true loss of the mixture policy.

\end{document}